\documentclass[nohyperref]{article}

\usepackage{microtype}
\usepackage{graphicx}
\usepackage{subfigure}
\usepackage{booktabs} % for professional tables

\usepackage{hyperref}

\usepackage[accepted]{icml2022}

\usepackage{amsmath}
\usepackage{amssymb}
\usepackage{mathtools}
\usepackage{amsthm}
\usepackage{enumitem}
\usepackage[capitalize,noabbrev]{cleveref}

\theoremstyle{plain}

\numberwithin{equation}{section}

\newtheorem{lemma}{Lemma}
\newtheorem{theorem}{Theorem}
\newtheorem{proposition}{Proposition}
\newtheorem{definition}{Definition}
\newtheorem{corollary}[theorem]{Corollary}

\ifx\assumption\undefined
\newtheorem{assumption}{Assumption}
\fi
\newcommand{\norm}[1]{\left\|#1\right\|}

\newcommand{\cD}{\mathcal{D}}

\newcommand{\cA}{\mathcal{A}}
\newcommand{\cB}{\mathcal{B}}
\newcommand{\cX}{\mathcal{X}}

\newcommand{\cE}{\mathcal{E}}

\newcommand{\cF}{{{\mathcal{F}}}}
\newcommand{\cT}{{\mathcal{T}}}

\newcommand{\Eb}{\mathbb{E}}
\newcommand{\RR}{\mathbb{R}}

\newcommand{\Pb}{\mathbb{P}}
\newcommand{\Db}{\mathbb{D}}

\newcommand{\argmin}{\mathop{\mathrm{argmin}}}
\newcommand{\argmax}{\mathop{\mathrm{argmax}}}

\newcommand{\Reg}{\mathop{\mathrm{Reg}}}

\usepackage[textsize=tiny]{todonotes}

\allowdisplaybreaks

\icmltitlerunning{A Self-Play Posterior Sampling Algorithm for Zero-Sum Markov Games}

\begin{document}

\twocolumn[
% \icmltitle{Policy Teaching for Bandit and Reinforcement Learning: \\ A Reinforcement Learning Perspective}
\icmltitle{A Self-Play Posterior Sampling Algorithm for Zero-Sum Markov Games}

% It is OKAY to include author information, even for blind
% submissions: the style file will automatically remove it for you
% unless you've provided the [accepted] option to the icml2022
% package.

% List of affiliations: The first argument should be a (short)
% identifier you will use later to specify author affiliations
% Academic affiliations should list Department, University, City, Region, Country
% Industry affiliations should list Company, City, Region, Country

% You can specify symbols, otherwise they are numbered in order.
% Ideally, you should not use this facility. Affiliations will be numbered
% in order of appearance and this is the preferred way.
%\icmlsetsymbol{equal}{*}

\begin{icmlauthorlist}
\icmlauthor{Wei Xiong}{xxx}
\icmlauthor{Han Zhong}{yyy}
\icmlauthor{Chengshuai Shi}{zzz}
\icmlauthor{Cong Shen}{zzz}
\icmlauthor{Tong Zhang}{xxx,fff}
%\icmlauthor{}{sch}
%\icmlauthor{}{sch}
\end{icmlauthorlist}

\icmlaffiliation{xxx}{The Hong Kong University of Science and Technology;}
\icmlaffiliation{yyy}{Center for Data Science, Peking University;}
\icmlaffiliation{zzz}{University of Virginia;}
\icmlaffiliation{fff}{Google Research}

\icmlcorrespondingauthor{Tong Zhang}{tongzhang@tongzhang-ml.org}

%\icmlcorrespondingauthor{Firstname2 Lastname2}{first2.last2@www.uk}

% You may provide any keywords that you
% find helpful for describing your paper; these are used to populate
% the "keywords" metadata in the PDF but will not be shown in the document
\icmlkeywords{Machine Learning, ICML}

\vskip 0.3in
]

% this must go after the closing bracket ] following \twocolumn[ ...

% This command actually creates the footnote in the first column
% listing the affiliations and the copyright notice.
% The command takes one argument, which is text to display at the start of the footnote.
% The \icmlEqualContribution command is standard text for equal contribution.
% Remove it (just {}) if you do not need this facility.

%\printAffiliationsAndNotice{}  % leave blank if no need to mention equal contribution
\printAffiliationsAndNotice{} % otherwise use the standard text.

\begin{abstract}
Existing studies on provably efficient algorithms for Markov games (MGs) almost exclusively build on the ``optimism in the face of uncertainty'' (OFU) principle. This work focuses on a different approach of posterior sampling, which is celebrated in many bandits and reinforcement learning settings but remains under-explored for MGs. Specifically, for episodic two-player zero-sum MGs, a novel posterior sampling algorithm is developed with \emph{general} function approximation. Theoretical analysis demonstrates that the posterior sampling algorithm admits a $\sqrt{T}$-regret bound for problems with a low multi-agent decoupling coefficient, which is a new complexity measure for MGs, where $T$ denotes the number of episodes. When specialized to linear MGs, the obtained regret bound matches the state-of-the-art results. To the best of our knowledge, this is the first provably efficient posterior sampling algorithm for MGs with frequentist regret guarantees, which enriches the toolbox for MGs and promotes the broad applicability of posterior sampling.

\end{abstract}
%We propose a new algo- rithm that can provably find the Nash equilibrium policy using a polynomial number of samples, for any MG with low multi-agent Bellman-Eluder dimension—a new complexity measure adapted from its single-agent version (Jin et al., 2021). A key component of our new algorithm is the booster agent, which facilitates the learning of the main agent by deliberately exploiting her weakness. Our the- oretical framework is generic, which applies to a wide range of models including but not limited to tabular MGs, MGs with linear or kernel function approximation, and MGs with rich observations.

\section{Introduction}\label{sec:intro}
Multi-agent reinforcement learning (MARL) focuses on the sequential decision-making problem involving more than one agent, each of which aims to optimize her own long-term return by interacting with the environment and other agents \cite{zhang2021multi}. Today, MARL has a diverse set of real-world applications, including Go \cite{silver2016mastering, silver2017mastering}, autonomous driving \cite{shalev2016safe}, Poker \cite{brown2019superhuman}, and Dota \cite{berner2019dota}, just to name a few. Due to the large state space of these practical problems, function approximation (with neural networks) is often used in these applications for the generalization across different state-action pairs. 
While there is a long line of related works on the theoretical understanding of single-agent RL with general function approximation \cite{jiang@2017, sun2019model, wang2020reinforcement, jin2021bellman, du2021bilinear, dann2021provably}, the theory of MARL with general function approximation is substantially less explored. In this paper, we aim to explore this topic in the context of two-player zero-sum Markov games (MGs) \cite{shapley1953stochastic, littman1994markov}.

The goal of learning in a two-player zero-sum MG is to learn the Nash equilibrium at which the policy of each player maximizes her own cumulative rewards, provided that the policies of other agents are fixed. Intuitively speaking, Nash equilibrium characterizes the point from which no agent will deviate. Since the reward and the state transition are determined jointly by the actions of both agents, in addition to the unknown environment, each agent must also handle the dynamics of other strategic agents. Due to this game-theoretical feature, algorithms designed for MDP cannot be directly extended to the MARL case. However, recent studies \cite{jin2021power, huang2021towards} have shown that with an innovative asymmetrical structure, similar theoretical results can be established for the two-player zero-sum MG with general function approximation.

Nevertheless, despite a handful of recent progress on the theory of the two-player zero-sum MG with general function approximation, the existing works are mainly confined to algorithms based on the optimism in the face of uncertainty (OFU) principle. In contrast, the theory of posterior-sampling-based algorithms is less developed (in the frequentist setting). Various empirical studies indicate that the OFU-based algorithms can be far too optimistic for average instances and is inferior to posterior sampling algorithms, including \citet{chapelle2011empirical} for bandit, and \citet{osband2016generalization} for RL. Recent works in the context of contextual multi-armed bandit and single-agent RL demonstrate that there is no statistical efficiency gap between OFU and posterior sampling algorithms \cite{dann2021provably, zhang2021feel}. However, whether we can design model-free posterior sampling algorithms in MARL that achieve similar theoretical guarantees remains open. 

In this paper, we are interested in the application of posterior sampling in the two-player zero-sum MG with general function approximation and self-play (which means that the learning agent can control both players). Our main result indicates that, similar to the single-agent case, posterior sampling algorithms can achieve comparable theoretical guarantees as to the OFU-based algorithms. Our contributions are summarized as follows:
\setlist{nolistsep}
\begin{itemize}[noitemsep,leftmargin=*]
    \item A provably efficient posterior sampling algorithm is designed under the self-play framework for the two-player zero-sum MG with general function approximation. To the best of our knowledge, this is the first posterior sampling algorithm with frequentist regret guarantee in the context of Markov games;
    \item The single-agent complexity measure of {\it decoupling coefficient}, first introduced in \citet{dann2021provably}, is extended to the multi-agent setting, Moreover, a number of examples with provably small multi-agent decoupling coefficients are identified;
    \item The proposed algorithm is rigorously proved to obtain a $\sqrt{T}$-regret for problems with low multi-agent decoupling coefficient, where $T$ is the number of episodes.
\end{itemize}

It is noted that the sampling procedure of the proposed algorithm may not be computationally efficient. The lack of computational tractability also appears in the works of \citet{jin2021power, huang2021towards}, as well as previous works with general function approximation in the context of single-agent RL \citep{jiang@2017, jin2021bellman, dann2021provably, du2021bilinear}. It is an interesting future research topic to identify cases where efficient sampling is possible. Moreover, we do not take credit for {the asymmetrical structure in our algorithmic framework}. The main contribution here is to extend the posterior sampling algorithm to MGs under the self-play framework. 

\subsection{Related Works}\label{sec:related}
There have been a lot of works focusing on designing provably efficient algorithms for zero-sum MGs. For the tabular setting, \citet{bai2020near,bai2020provable,liu2020sharp} provide $O(\text{poly}(|\cX|, |\cA|, |\cB|, H) \cdot \sqrt{T})$ regret guarantees for the proposed algorithms, where $|\cX|$ is the number of states, $|\cA|$ and $|\cB|$ are the number of action spaces of two players, respectively, $H$ is the episode length, and $T$ is the number of episodes. Then, \citet{xie2020learning, chen2021almost} study two linear-type MGs and design algorithms with $O(\text{poly}(d, H) \cdot \sqrt{T})$ regret, where $d$ is the dimension of the linear features. Recently, \citet{jin2021power,huang2021towards} further propose efficient algorithms for zero-sum MGs with general function approximation. 

Our work is also closely related to another line of work on posterior sampling algorithms. In the context of contextual bandit, due to the impressive empirical performance of Thompson Sampling \citep{chapelle2011empirical, osband2016generalization}, there have been significant efforts in developing its theoretical analysis, including \citet{russo2014learning} in the form of Bayesian regret and \citet{kaufmann2012thompson, zhang2021feel} in the frequentist setting. For the Markov Decision Process (MDP), the seminal work \citet{osband2014model} considers the Bayesian regret and proposes a general posterior sampling RL method. The randomized least-squares value iteration (RLSVI) algorithm \cite{osband2016generalization} is shown to admit frequentist regret bounds for tabular MDP \cite{russo2019worst, agrawal2020improved, xiong2021randomized} and linear MDP \cite{zanette2020frequentist}. Beyond the linear setting, a recent work \cite{dann2021provably} proposes a conditional posterior sampling algorithm to solve the MDP with general function approximation.

A recent posterior-sampling-type work by \citet{jafarnia2021learning} considers the infinite-horizon zero-sum MGs with average-reward criterion in the tabular setting, with a focus on the Bayesian regret, whose analysis technique is fundamentally different from ours. To the best of our knowledge, there is no posterior sampling algorithm with a frequentist regret guarantee to date for MGs.

\section{{Problem Formulation}} \label{sec:pre}
Markov Games (MGs) generalize the standard Markov Decision Processes to the multi-agent setting. In this work, the episodic two-player zero-sum MG is considered, which can be formally denoted as $MG(H, \cX, \cA, \cB, \Pb, r)$.
Here $H$ denotes the length of each episode, $\cX$ is the (possibly infinite) state space, $\cA$ and $\cB$ are the action spaces of two players (referred to as the max-player and the min-player), respectively, $\Pb_h(\cdot|x,a,b)$ is the transition measure of the next state from the current state $x$ with two actions $(a,b)$ taken at step $h$, and $r^h(x,a,b)$ is the {corresponding} reward received with actions $(a,b)$ taken for state $x$ at step $h$. 

{Specifically, in this MG, each episode $t$ starts from an initial state $x_t^1$. At each step $h$, two players observe the current state $x_t^h$, take actions $(a^h_t, b^h_t)$ individually, and observe the next state $x^{h+1}_t \sim \Pb_h(\cdot|x_t^h, a_t^h, b_t^h)$. The current episode ends after step $H$ and then a new episode starts. Without loss of generality, each episode is assumed to have a fixed initial state $x_t^1 = x^1$, which can be easily generalized to having $x_t^1$ sampled from a fixed but unknown distribution. }

{Also, for the ease of presentation, the reward $r^h(x,a,b)$ is assumed to be deterministic and in the interval of $[0,1]$ for any $(x,a,b)$ in this paper, while the algorithm designs and theoretical results can also be applied for stochastic bounded rewards with slight modifications.}  %We assume that each episode starts with a fixed initial state $x_t^1 = x^1$ and this can be simply generalized to the case where $x_t^1$ is sampled from some fixed but unknown distribution. 

{\textbf{Policies and Value Functions.}} {With $\Delta_\cA$ denoting the probability simplex over the action space $\cA$,} a Markov policy $\mu$ of the max-player can be defined as $\mu := \{\mu_h: \cX \to \Delta_{\cA}\}_{h \in [H]}$. Similarly, we can define a Markov policy $\nu:=\{\nu_h:\cX\to \Delta_{\cB}\}_{h\in [H]}$ for the min-player.

Given a policy pair $(\mu, \nu)$, {the value function $V_h^{\mu, \nu}: \cX \to \RR$ at step $h$ is defined as
\begin{equation*}
    V_{h}^{\mu, \nu}(x):=\mathbb{E}_{\mu, \nu}\left[\sum_{h'=h}^{H} r^{h'}\left(x^{h'}, a^{h'}, b^{h'}\right) \mid x^{h}=x\right]
\end{equation*}
and the Q-value function $Q_h^{\mu,\nu}: \cX \times \cA \times \cB \to \RR$ as
\begin{align*}
Q_{h}^{\mu, \nu}(x, a, b):=\mathbb{E}_{\mu, \nu}\bigg[\sum_{h=h}^{H} r^{h'}\left(x^{h'}, a^{h'}, b^{h'}\right)  \mid \\(x^{h}, a^h, b^h) = (x,a,b)\bigg],\notag
\end{align*}}%
where the expectations are taken over the randomness of the environment and the policies. 

For a clean presentation, we use the notation $\Pb_h$ (with a slight abuse) so that $[\Pb_h V](x,a,b) = \Eb_{x' \sim \Pb_h(\cdot|x,a,b)} V(x')$ for any value function $V$. Similarly, the notation $\Db_\pi$ is adopted so that 
$
\left[\mathbb{D}_{\pi} Q\right](x):=\mathbb{E}_{(a, b) \sim \pi(\cdot, \cdot \mid x)} Q(x, a, b),
$ for any policy pair $\pi=(\mu, \nu)$ and action-value function $Q$. With these notations, the Bellman equations are given by
\begin{align*}
Q_{h}^{\mu, \nu}(x, a, b)&=\left[r^{h}+\mathbb{P}_{h} V_{h+1}^{\mu, \nu}\right](x, a, b);\\
V_{h}^{\mu, \nu}(x)&=\left[\mathbb{D}_{\mu_{h} \times \nu_{h}} Q_{h}^{\mu, \nu}\right](x).
\end{align*}

{\textbf{Best Response.}} For any policy of max-player $\mu$, {a corresponding best response for the min-player can be found, denoted as} $\nu^\dagger(\mu)$, such that $V_h^{\mu, \nu^\dagger(\mu)}(x) = \inf_\nu V_h^{\mu, \nu}(x)$ for all $(x,h)$.  This value is the best favorable result for the min-player if the max-player announces that she will play strategy $\mu$. Similarly, for a min-player policy $\nu$, there exists a best response for the max-player, denoted as $\mu^\dagger(\nu)$, such that $V_h^{\mu^\dagger,\nu}(x) = \sup_{\mu} V_h^{\mu, \nu}(x)$ for all $(x,h)$. To simplify the notation, we use
\begin{align*}
V_h^{\mu,\dagger}(x) := V_h^{\mu, \nu^\dagger(\mu)}(x), Q_h^{\mu,\dagger}(x,a,b) := Q_h^{\mu, \nu^\dagger(\mu)}(x,a,b);\\
V_h^{\dagger, \nu}(x) := V_h^{\mu^\dagger(\nu),\nu}(x),Q_h^{\dagger, \nu}(x,a,b) := Q_h^{\mu^\dagger(\nu),\nu}(x,a,b).
\end{align*}

{\textbf{Nash Equilibrium.}} Moreover, there exists {a set of Nash equilibrium (NE) policies $(\mu^*, \nu^*)$} \cite{filar2012competitive} that are optimal against their best response such that 
\begin{equation*}
    V_{h}^{\mu^{*}, \dagger}(x)=\sup\nolimits_{\mu} V_{h}^{\mu, \dagger}(x), \quad V_{h}^{\dagger, \nu^{*}}(x)=\inf\nolimits_{\nu} V_{h}^{\dagger, \nu}(x),
\end{equation*}
for all $(x,h) \in \cX \times [H]$. For this NE, the following famous minimax equation holds:
\begin{equation*}
    \sup\nolimits_{\mu} \inf\nolimits_{\nu} V_{h}^{\mu, \nu}(x)=V_{h}^{\mu^{*}, \nu^{*}}(x)=\inf\nolimits_{\nu} \sup\nolimits_{\mu} V_{h}^{\mu, \nu}(x)
\end{equation*}
for all $(x, h) \in \cX \times [H]$. For simplicity, we denote $V_h^*(x) := V_h^{\mu^*, \nu^*}(x)$ and $Q_h^*(x) := Q_h^{\mu^*, \nu^*}(x)$. Note that although there might exist multiple NE policies, the NE value function is unique for a zero-sum MG.

\textbf{Performance metrics.}
%In the two-player MG, the final goal for each player would be to output a policy approximating the Nash equilibrium. During the training, a natural and well-adopted performance measure is the cumulative regret against the NE. In this work, the perspective of training the max-player is taken and her regret over an overall $T$ episodes can be defined as
{A max-player's policy $\mu$ is said to be $\epsilon$-close to the NE if it satisfies $V^*(x^1) - V^{\mu,\dagger}(x^1) < \epsilon$. Note that we have $V^*(x^1) - V^{\mu,\nu}(x^1) \leq V^*(x^1) - V^{\mu,\dagger}(x^1)$ for all min-player's policy $\nu$ as the best response is the strongest opponent for the max-player. The main goal of this paper is to find an $\epsilon$-close policy for the max-player and her regret over $T$ episodes can be defined as}
\begin{align*}
\Reg(T):=\sum_{t=1}^{T}\left[V_{1}^{*}\left(x_{1}\right)-V_{1}^{\mu_t, \dagger}\left(x_{1}\right)\right],
\end{align*}
{where $\mu_t$ is the policy adopted by the max-player for episode $t$. Note that we can switch the roles of two players to learn a policy $\nu$ that is $\epsilon$-close to the NE for the min-player.}

\subsection{Function Approximation}
As mentioned in Sec.~\ref{sec:intro}, real-world applications of RL often encounter the challenge of a large state space where storing a table as in the classical Q-learning is generally infeasible. To overcome this challenge, function approximation is proposed and proven to be efficient with many practical successes. Following similar attempts in MDP, we aim to approximate {the $Q$-value functions for the MGs considered in this work} by a class of functions $\cF = \cF_1 \times \cdots \times \cF_H$ where $\cF_h \subset (\cX \times \cA \times \cB \to \RR)$. 

For $f \in \cF$, a NE can be induced and the corresponding policy $\mu_f$ of the max-player is defined for all $(x, h)$ as
\begin{equation*}\small
    \mu_{f, h}(x)=\argmax\nolimits_{\mu \in \Delta_{\mathcal{A}}}\min\nolimits_{\nu \in \Delta_{\mathcal{B}}} \mu^{\top} f^{h}(x, \cdot, \cdot) \nu.
\end{equation*}
{The induced value function for all $(x, h)$ is then given by}
\begin{equation*}
    V_{f, h}(x)=\max\nolimits_{\mu \in \Delta_{\mathcal{A}}} \min\nolimits_{\nu \in \Delta_{\mathcal{B}}} \mu^{\top} f^{h}(x, \cdot, \cdot) \nu.
\end{equation*}

Moreover, for a fixed max-player policy $\mu$ and a function $f\in \cF$, the induced value function of the best response of the min-player is defined  for all $(x, h)$ as 
\begin{equation*}
    V_{f, h}^{\mu}(x)=\min\nolimits_{\nu \in \Delta_{\mathcal{B}}} \mu_{h}(x)^{\top} f^{h}(x, \cdot, \cdot) \nu.
\end{equation*}
This is mainly for the min-player to choose her policy, given the max-player's policy $\mu$. Details can be found in Sec.~\ref{sec:alg}.

As common in \citet{perolat2015approximate, jin2021power, huang2021towards}, two types of Bellman operators are defined as
{
\begin{align*}
    \left(\mathcal{T}_{h} f\right)(x, a, b)&:=\left[r^{h}+\Pb_{h} V_{f, h+1}\right](x, a, b); \\
    \left(\mathcal{T}_{h}^{\mu} f\right)(x, a, b)&:=\left[r^{h}+\Pb_h V_{f, h+1}^{\mu}\right](x, a, b).
\end{align*}}%
% \begin{align*}
%     \left(\mathcal{T}_{h} f\right)(x, a, b)&:=r^{h}(x, a, b)+\mathbb{E}_{x^{\prime} \sim \mathbb{P}_{h}(\cdot \mid x, a, b)} V_{f, h+1}\left(x^{\prime}\right); \\
% \left(\mathcal{T}_{h}^{\mu} f\right)(x, a, b)&:=r^{h}(x, a, b)+\mathbb{E}_{x^{\prime} \sim \mathbb{P}_{h}(\cdot \mid x, a, b)} V_{f, h+1}^{\mu}\left(x^{\prime}\right).
% \end{align*}
The corresponding Bellman residual are denoted as
\begin{equation}
    \label{eqn:bellman_residual}
    \begin{aligned}
    \cE_h(f;x,a,b) &= \cE(f^h,f^{h+1};x,a,b) \\
    &= f^h(x,a,b) - (\cT_h f)(x,a,b);\\
    \cE_h^{\mu}\left(f ; x, a, b\right) &= \cE^{\mu}(f^h,f^{h+1};x,a,b)\\
    &=f^h(x,a,b) - (\cT^{\mu}_h f)(x,a,b).
    \end{aligned}
\end{equation}
Sometimes the state-action pair $(x,a,b)$ {may be replaced} with a trajectory $\zeta = \{(x^{h'}, a^{h'}, b^{h'}, r^{h'})\}_{h'=1}^H$, which indicates that the corresponding state-action pair at step $h$, i.e. $(x^{h}, a^{h}, b^{h})$, is taken as input. 

Recent advances show that RL with function approximation is, in general, intractable without any further assumption \citep{krishnamurthy2016pac,weisz2021exponential}. It is thus common to adopt additional assumptions over the function class in the literature on general function approximation in MDPs, especially the realizability and completeness assumptions  \citep{wang2020reinforcement,jin2021bellman, dann2021provably}. {As MGs are natural extensions of MDPs, the generalized realizability and completeness assumptions are also adopted in this work. Note that Assumptions~\ref{assu:realizability} and \ref{assu:completeness}  are also required by other recent works on MGs with general function approximation \citep{jin2021power, huang2021towards}.}
\begin{assumption}[Realizability] \label{assu:realizability}
{For the Nash equilibrium, it holds} that $
    Q^*_h \in \cF_h, \forall h \in [H]$.
Moreover, for any $f \in \cF$, it holds that $
    Q_h^{\mu_f, \dagger} \in \cF_h, \forall h\in [H]$.
\end{assumption}
The realizability assumption states that the function class $\cF$ is large enough so that it contains the $Q$-value function of the NE and also the $Q$-value function of any induced policy and its best response. 

%The completeness assumption is stronger than the realizability as it implies realizability by backward inductions starting from the last step. 
The completeness assumption is more restrictive where the main drawback is that completeness is non-monotone, meaning that adding one function into $\cF$ may violate the assumption. However, it is the key to handling the variance of sampling in the literature and analysis without completeness seems very challenging.

\begin{assumption}[Completeness] \label{assu:completeness}
{For any $f, g\in \cF$ and the induced policy $\mu_f$, it holds that} $
    \cT_h^{\mu_f} g \in \cF_h, \forall h\in [H]$.
\end{assumption}

Additionally, the following boundedness assumption\footnote{$(\beta-1)$ is usually assumed to be either $1$ or $H$ in the literature.} is considered, which is natural for bounded rewards and a finite episode length.
\begin{assumption}[Boundedness] \label{assu:bounded}
There exists $\beta > 1$ s.t. $f^h(x,a,b) \in [0, \beta-1], \forall (f,h,x,a,b) \in \cF \times [H] \times \cX \times \cA \times \cB$. 
\end{assumption}

\section{Algorithm}\label{sec:alg}
The proposed {\it Conditional Posterior Sampling with Booster} algorithm is presented in this section. ''Conditional'' refers to the design of $q(f^h|f^{h+1}, S_t)$ (the denominator in Eqn.~\eqref{eqn:likelihood}), which allows us to use the true Bellman operator $\cT_h$ in the analysis even though we do dot know it in the executed algorithm. ''Booster'' is a synonym for exploiter \citep{jin2021power} in the context and refers to the asymmetric structure as the second agent aims to assist the main agent's learning. Another reason is that when we wrote this paper, one of the authors had a fever due to the booster vaccine, and another author tested positive for covid.

\begin{algorithm}[tbh]
	\caption{Conditional Posterior Sampling with Booster}
	\label{alg:mg_TS}
	\begin{small}
		\begin{algorithmic}[1]
		    \STATE \textbf{Input}: function class: $\cF$, learning rate $\eta$, horizon $T$, prior parameter $\lambda$.
		    \STATE $S_0$ is initialized to be empty.
			\FOR{Stage $t=1,\dots,T$}
			%\STATE Observe $x^1 \in \cX$;
			\STATE Main agent: $\mu_t\gets $ Main($\cF, \eta, S_{t-1}, T, \lambda$);
			\STATE Booster agent: $\nu_t\gets$ Booster($\cF, \eta, S_{t-1}, \mu_t, T, \lambda$);
			%\STATE Draw $f_t \sim p(\cdot|S_{t-1})$ and take $\mu_t = \mu_{f_t}$;
			%\STATE Take $\nu_t$ = booster agent($\cF$, $\cX_{t-1}$, $\mu_t$, $\lambda$);
			\STATE Execute the policies $(\mu_t, \nu_t)$ and collect the trajectory $(x_t^1, a_t^1, b_t^1, r_t^1, \cdots, x_t^H, a_t^H, b_t^H, r_t^H)$ to obtain $S_t$.
			\ENDFOR
		\end{algorithmic}
	\end{small}
\end{algorithm}

\subsection{Overview}

%\shic{It may be better to start the overview with a few sentences of successes in OFU-based or frequentist algorithms in MGs. Then, briefly introduce the basic structure of TS-based or Bayesian algorithms, and the recent advances with them in single-agent RL. After that, illustrate the lack of research in TS with MGs and highlight on the difficulties. Some difficulties are already mentioned in the current draft, but it would be even better to find some unique ones for TS. Or simply state that the Bayesian nature of TS couples with the complicated MG.}

%\shic{Maybe we can mention the use of booster agent but not in the overview. In the booster agent subsection, we can say that inspired by Chi Jin's work, an asymmetrical structure is adopted and the second player is designed to help the main agent in finding the Nash equilibrium. However, highlight that we are designing a Bayesian algorithm. Thus, although the design of the second player is philosophically similar to the OFU booster agent, the underlying techniques are fundamentally different.}
%\weir{Major modification here.}
The existing algorithms with frequentist guarantees are confined to OFU-based algorithms. Algorithmically, these algorithms typically maintain a confidence set $\mathcal{C}$ whose components are empirically consistent with the Bellman equation so far. Then, an optimistic function $f \in \mathcal{C}$ is selected to approximate the true value function through some optimization subroutine \cite{jin2021power, huang2021towards}. In contrast, the posterior sampling algorithm starts with a \textit{prior} $p_0(\cdot)$ over the function class $\cF$ and collects trajectories to compute the \textit{likelihood}; they together lead to a posterior distribution $p(\cdot)$ over $\cF$. Then, a function is sampled from the posterior distribution to approximate the target. In addition to the difference in algorithm structure, the analysis techniques for the posterior sampling algorithm are also different, particularly due to the lack of explicit optimism from the planning step.

A frequentist theoretical guarantee of posterior sampling algorithms has been lacking for a long time, even in the context of contextual bandit. Recently, \citet{zhang2021feel} and \citet{dann2021provably} show that adding an extra optimistic term can lead to frequentistly optimal posterior sampling algorithm in contextual bandit and MDP, respectively. However, in the MARL setting, the multi-agent nature leads to complicated statistical dependence across the players. In particular, in addition to the environment, the agent will also be affected by other strategic agents. Therefore, the situation is more complicated even in the two-player case and the algorithms designed for MDPs cannot be directly extended to MGs. To overcome this issue, inspired by \citet{jin2021power, huang2021towards}, we leverage the innovative asymmetric structure to pick the max-player and the min-player as the main agent and the booster agent, respectively, where the booster agent, as the name suggests, aims to assist the main agent's learning. 

Our algorithm is summarized in Algorithm~\ref{alg:mg_TS} where the main agent's algorithm and the booster agent's algorithm are given in Algorithms~\ref{alg:main} and \ref{alg:booster agent}, respectively.

\begin{algorithm}[htb]
	\caption{Main($\cF, \eta, \cD, T, \lambda$)}
	\label{alg:main}
	\begin{small}
		\begin{algorithmic}[1]
			\STATE Draw $f \sim p(\cdot|\cD)$ where the posterior is given by Eqn.~\eqref{eqn:posterior_main};
			\STATE $\mu_{f, h}(x)=\argmax\nolimits_{\mu \in \Delta_{\mathcal{A}}} \min _{\nu \in \Delta_{\mathcal{B}}} \mu^{\top} f^{h}(x, \cdot, \cdot) \nu, \forall (x, h)$;
			\STATE Return $\mu_{f}$.
    \end{algorithmic}
	\end{small}
\end{algorithm}
\subsection{The Main Agent}
{The main agent's goal is to learn a $\epsilon$-close policy for the max-player, i.e., $V^*(x^1) - V^{\mu,\dagger}(x^1) < \epsilon$. With function class $\cF$ available, she aims to find a function $f\in \cF$ to approximate the Nash $Q$-value function, i.e., $Q^*$, which can be used to solve the Nash policy via the minimax equation. The following optimistic prior and temporal difference error likelihood are carefully crafted to induce a desired posterior distribution over $\cF$, which is further used to sample a suitable function $f$. }

\textbf{Optimistic prior.} The following prior $\tilde{p}_0(\cdot)$ over the function class $\cF$ is adopted for the main agent:
\begin{equation}
    \label{eqn:prior}
    \tilde{p}_0(f) \propto \exp(\lambda V_{f, 1}(x^1)) \prod_{h=1}^H p_0^h(f^h),
\end{equation}
where $\lambda > 0$ is a tuning parameter, and $p_0^h(\cdot)$ is a distribution over $\cF_h$. Note that other than the standard prior of $p_0(f)=\prod_{h=1}^H p_0^h(f^h)$, an additional \emph{optimistic} term, i.e., $\exp(\lambda V_{f,1}(x^1)$, is involved in the prior, which plays an important role of  encouraging exploration for the main agent.

This prior is referred to as an optimistic one because it favors large values for the initial state. Also, technically, it compensates for one extra term arising in the value decomposition in Lemma~\ref{lem:value_decom_main} when the optimism is {not inherently available as in OFU-based algorithms}. {Similar techniques are also adopted in the design of posterior sampling for MDPs \citep{dann2021provably} and contextual bandits \citep{zhang2021feel}.} Furthermore, \citet{zhang2021feel} argues that in the context of contextual bandit, such an optimistic component is necessary to design optimal posterior-sampling-based algorithms in the frequentist setting.

Also, apart from the optimism itself, the \textit{global} optimism mechanism, meaning that we only add an optimistic term in the prior distributions at the initial value, is the key to achieving improvement in the feature dimension for linear MGs. We will return to this in Sec.~\ref{sec:illu}.

\textbf{Likelihood for the main agent.} If we denote the history up to the end of episode $t$ as $S_t = \{x_s^h, a_s^h, b_s^h, r_s^h\}_{s \in [t], h \in [H]}$, a likelihood over $S_t$ is specified as
\begin{equation}\label{eqn:likelihood}\small
    p(S_t|f) \propto \prod_{h=1}^H \frac{\exp
    \left(- L^h(f^h, f^{h+1};S_t)\right)}{\Eb_{f^h \sim p_0^h} \exp(-\eta L^h(f^h, f^{h+1}; S_t))}.
\end{equation}
$\{L^h(\cdot)\}_{h=1}^H$ is a collection of {squared} loss functions as
\begin{equation*}\small
\begin{aligned}
        L^h(f^h, f^{h+1}; S_t) = \sum_{s=1}^{t} \left[f^h(x_s^h,a^h_s,b^h_s) - r_s^h - V_{f^{h+1}}(x^{h+1}_s)\right]^2,
\end{aligned}
\end{equation*}
which is a proxy to the squared $\cT_h$-Bellman error. The likelihood in Eqn.~\eqref{eqn:likelihood} introduces a special denominator, which is motivated by that for MDP \citep{dann2021provably}. We will see that the denominator is the key to handling the variance of sampling, but that is also why we need the strong completeness assumption. We will discuss this in Sec.~\ref{sec:illu_likelihood}.

\textbf{Posterior distribution for the main agent.} Given the prior distribution and the likelihood, the posterior at the end of episode $t$ can be naturally expressed as 
\begin{equation}
    \label{eqn:posterior_main}
    \begin{aligned}
    {p}(f|S_t) \propto \exp(\lambda V_{f,1}(x^1)) \prod_{h=1}^H q(f^h|f^{h+1}, S_t), 
    \end{aligned}
\end{equation}
where 
\begin{equation*}\small
\begin{aligned}
        q(f^h|f^{h+1}, S_t) = \frac{p_0^h(f^h)\exp
    \left(- \eta L^h(f^h, f^{h+1}; S_t)\right)}{\Eb_{f^h \sim p_0^h} \exp(-\eta L^h(f^h, f^{h+1}; S_t))}.
\end{aligned}
\end{equation*}

%One main difference for the selections of $\mu_t$ and $\nu_t$ is on the definition of $L^h$. The main agent will adapt a squared loss function defined as 
%\begin{equation}
%     L^h(f^h, f^{h+1};S_t) = \sum_{s=1}^{t} [f^h(x_s^h,a^h_s,b^h_s) - r_s^h - V_{f^{h+1}}(x^{h+1}_s)]^2,
% \end{equation}
%which can be viewed as a proxy to the squared $\cT_h$-Bellman error.
\begin{algorithm}[thb]
	\caption{Booster($\cF, \eta, \cD, \mu_{f}, T, \lambda$)}
	\label{alg:booster agent}
	\begin{small}
		\begin{algorithmic}[1]
			\STATE Draw $g \sim p^{\mu_{f}}(\cdot|\cD)$ where the posterior is given by Eqn.~\eqref{eqn:posterior_booster agent}
			\STATE 
			$\nu_{h}(x) = \nu_{f,g,h}(x) = \argmin_{\nu \in \Delta_{\cB}} \mu_{f,h}^\top g^h(x,\cdot,\cdot)\nu,\forall (x,h)$
            \STATE Return $\nu$.
    \end{algorithmic}
	\end{small}
\end{algorithm}
\subsection{The Booster Agent} 
As aforementioned, the main agent aims to learn an $\epsilon$-close policy. However, given the competing nature of MGs, this task is not feasible if her opponent is naive. Thus, inspired by \citet{jin2021power, huang2021towards}, the second learning agent is set to be the booster agent. As opposed to the main agent, the booster agent does not aim at find her $\epsilon$-close policy. Instead, her goal is to assist the main agent's learning. Specifically, she examines the adopted policy of the main agent and tries to find the best response for it (since the best response is the strongest opponent). In this way, the underlying weakness of the main agent is exploited, which facilitates the learning of the NE. To better illustrate the role of the booster agent, we consider the following decomposition of the regret:
\begin{equation}\small
\label{eqn:decom}
\begin{aligned}
    \Reg(T)%=&\sum_{t=1}^T (V_1^*(x_1) - V_1^{\mu_t,\dagger}(x_1)) \\
    =& \underbrace{\left(\sum_{t=1}^TV_1^*(x^1) - V_1^{\mu_t, \nu_t}(x^1)\right)}_{\text{main agent}}\\
    &+ \underbrace{\left(\sum_{t=1}^TV_1^{\mu_t,\nu_t}(x^1) - V_1^{\mu_t, \dagger}(x^1)\right)}_{\text{booster agent}}.
\end{aligned}
\end{equation}
The technical advantage of involving $V_1^{\mu_t, \nu_t}(x^1)$ in the main agent part is that we can apply the value-decomposition lemma from \citet{jiang@2017} as in Lemma~\ref{lem:value_decom_main} because $(\mu_t, \nu_t)$ is the executed policy pair for trajectory collection (see Lemma~\ref{lem:value_decom_main} for details). In this case, the non-negative booster agent part is zero if we can find the best response to $\mu_t$ exactly. Motivated by this observation, the booster agent keeps learning to approximate the best response to the given max-player's policy based on the historical trajectories so as to minimize the booster agent part. Due to the different goals, the design philosophy of the booster agent is different from that of the main agent. Especially, she takes a different but also optimistic prior (for the min-player) and a different format of the likelihood. 

\textbf{Optimistic prior of the booster agent.} An optimistic prior is adopted for the booster agent, defined as
\begin{equation}
    \label{eqn:booster agent_prior}
    p^\mu_0(g) \propto \exp(-\lambda V^\mu_{g, 1}(x^1)) \prod_{h=1}^H p_0^h(g^h).
\end{equation}
Intuitively, the booster agent favors small values for the initial state, which is optimistic for the min-player. The motivation for such an optimistic prior will be clearer after the value decomposition lemma, i.e., Lemma~\ref{lem:value_decom_booster agent}, is presented. The reason why we only modify the prior will also be illustrated in Sec.~\ref{sec:illu}.

\textbf{Likelihood for the booster agent.} {As the booster agent} mainly focuses on approximating the best response policy to $\mu$ instead of finding NE, a different squared loss function is specified as:
\begin{equation*}\small
\begin{aligned}
        L^h_{\mu}(g^h,g^{h+1};S_t) = \sum_{s=1}^t [g^h(x_s^h,a^h_s,b^h_s) - r_s^h - V^\mu_{g^{h+1}}(x^{h+1}_s)]^2,
\end{aligned}
\end{equation*}
which can be viewed as a proxy to the squared $\cT_h^\mu$-Bellman error. {Consequently, a corresponding likelihood can be obtained by replacing $L^h$ in Eqn.~\eqref{eqn:likelihood} with $L^h_{\mu}$.}

\textbf{Posterior distribution for the booster agent.} With the prior and the likelihood, the posterior distribution for the booster agent can be obtained as:
\begin{equation}
    \label{eqn:posterior_booster agent}
    p^\mu(g|S_t) \propto \exp(-\lambda V^\mu_{g, 1}(x^1)) \prod_{h=1}^H q^\mu(g^h|g^{h+1}, S_t),
\end{equation}
where 
\begin{equation*}\small
\begin{aligned}
    q^\mu(g^h|g^{h+1}, S_t) = \frac{p_0^h(g^h)\exp
    \left(- \eta L^h_\mu(g^h, g^{h+1}; S_t)\right)}{\Eb_{g^h \sim p_0^h} \exp(-\eta L^h_\mu(g^h, g^{h+1}; S_t))}.
\end{aligned}
\end{equation*}
Note that sometimes we also employ the notation $q(g^h|g^{h+1}, \mu, S_t) = q^\mu(g^h|g^{h+1}, S_t)$ when we need to use the superscript $h$.

\subsection{The Learning Process} \label{sec:process}
{With the main agent and the booster agent specified, the training proceeds as the following.} For each episode $t$, the main agent first samples one $f_t \in \cF$ according to the posterior distribution $p(\cdot|S_{t-1})$ and adopts the induced Nash policy as %$\mu_{f_t,h}(x)=\argmax_{\mu \in \Delta_{\mathcal{A}}} \min _{\nu \in \Delta_{\mathcal{B}}} \mu^{\top} f_t^{h}(x, \cdot, \cdot) \nu$ for all $(x, h)\in \cX\times [H]$. 
\begin{equation*}\small
\begin{aligned}
        \mu_{t,h}(x)\gets \mu_{f_t,h}(x):=\argmax\nolimits_{\mu \in \Delta_{\mathcal{A}}} \min\nolimits_{\nu \in \Delta_{\mathcal{B}}} \mu^{\top} f_t^{h}(x, \cdot, \cdot) \nu
\end{aligned}
\end{equation*}
for all $(x, h)\in \cX\times [H]$.

{Then, the booster agent samples some $g_t \in \cF$ from her posterior distribution $p^{\mu_t}(\cdot|S_{t-1})$ computed from $S_{t-1}$ and $\mu_{t}$. The approximated best response is computed according to $g_t$ and $\mu_t$ as}
\begin{equation*}\small
    \nu_{t,h}(x) \gets \nu_{f_t,g_t,h}(x) = \argmin\nolimits_{\nu \in \Delta_{\cB}} \mu_{f_t,h}^\top g^h_t(x,\cdot,\cdot)\nu
\end{equation*}

for all $(x, h)\in \cX\times [H]$.

{Finally, both players execute $(\mu_t, \nu_t)$ for episode $t$, resulting in a trajectory $\zeta_t$. This collected trajectory is then added to $S_t$ and used to compute the policy for the next episode.}

\section{Sketch of the Main Ideas}\label{sec:main_idea}
In this section, a sketch of the main ideas is provided for the proposed algorithm and the theoretical proof.

\subsection{Value-Decomposition Lemmas}
It is known that the immediate regret in one episode can be related to the Bellman residuals in the single-agent setting \cite{jiang@2017}, and this technique is well-adopted in the literature \cite{jin2021bellman, dann2021provably, du2021bilinear}. For our setting, with regret decomposed as in Eqn.~\eqref{eqn:decom}, the immediate regrets of the main agent part and the booster agent part can be related to the $\cT$-Bellman residuals and the $\cT^{\mu_t}$-Bellman residuals, respectively, as we show below.
\begin{lemma}[Value decomposition for the main agent.] \label{lem:value_decom_main}
Let $\mu = \mu_f$ and $\nu$ be an arbitrary policy taken by the min-player. It holds that
\begin{equation*}\small
\begin{aligned}
&V^*(x^1) - V_1^{\mu, \nu}(x^1)\\
&\leq \sum_{h=1}^H \Eb_{\mu, \nu} \cE_h(f^h, f^{h+1}; \zeta) + V^*(x^1) - V_{f,1}(x^1).
\end{aligned}
\end{equation*}
\end{lemma}

\begin{lemma}[Value decomposition for the booster agent.]\label{lem:value_decom_booster agent}
Suppose that $\mu = \mu_f$ is taken by the max-player and $g$ is sampled from the posterior by the booster agent. Let $\nu$ be taken as in Sec.~\ref{sec:process}. Then, it holds that
\begin{equation*}\small
\begin{aligned}
&V_1^{\mu, \nu}(x^1) - V_1^{\mu, \dagger}(x^1)\\
&= -\sum_{h=1}^H \Eb_{\mu, \nu} \cE^\mu_h(g^h, g^{h+1},\zeta) + V_{g,1}^{\mu}(x^1) - V_1^{\mu, \dagger}(x^1).
\end{aligned}
\end{equation*}
\end{lemma}
We remark that these two lemmas also account for the extra optimistic terms in the prior distributions. The proofs of these two lemmas are deferred to Appendix~\ref{sec:value_decom}.

\subsection{Multi-Agent Decoupling Coefficients}
In the previous subsection, we convert the problem of bounding $\Reg(T)$ to bounding the summation of Bellman residuals. However, the posterior distribution is more related to the \emph{squared} Bellman residuals. Therefore, we need some structural information to relate the growth of the cumulative Bellman residuals to the growth of the cumulative squared Bellman residuals. To this end, the multi-agent decoupling coefficient is introduced, which is an extension of the single-agent version in \citet{dann2021provably}, as follows.

\begin{definition}[Multi-agent decoupling coefficient]\label{def:dc} Given an $MG(H, \cX, \cA, \cB, \Pb, r)$, a function class $\cF$, a time horizon $T$, and a parameter $\mu > 0$, the multi-agent decoupling coefficient $dc(\cF, MG, T, \mu)$ is defined to be the smallest integer such that
\begin{equation*}\small
\begin{aligned}
&\sum_{h=1}^{H} \sum_{t=1}^{T}\left[\mathbb{E}_{\pi_{{t}}}\left[\mathcal{E}_{h}^{\mu_{f_t}}\left(g_{t} ; x^{h}, a^{h}, b^h\right)\right]\right] \\
&\leq \mu \sum_{h=1}^{H} \sum_{t=1}^{T}\left[\sum_{s=1}^{t-1} \mathbb{E}_{\pi_{{s}}}\left[\cE_h^{\mu_{f_t}}\left(g_{t} ; x^{h}, a^{h}, b^h\right)\right]^2\right]+\frac{K}{4 \mu},
\end{aligned}
\end{equation*}
where $\pi_s$ is a policy pair $(\mu_{f_s}, \nu_{f_s,g_s})$ induced by $(f_s,g_s)$ as introduced in Sec.~\ref{sec:process}.
The set of these distributions induced by $f,g \in \cF$ is denoted as $\cD_\cF$.
\end{definition}

Equipped with the multi-agent decoupling coefficient, it remains to bound the cumulative squared Bellman residuals $\sum_{s=1}^{t-1} \left[\mathbb{E}_{\pi_{{s}}}\cE_h^{\mu_{f_t}}\left(g_{t} ; x^{h}, a^{h}, b^h\right)\right]^2$ by connecting it to the likelihood $L^h_{\mu_{f_t}}(g^h,g^{h+1};S_{t-1})$ used in the posterior distributions. 

\subsection{Connection to Likelihood} \label{sec:illu_likelihood}
We focus on the main agent and the booster agent is similar. We consider the $L^h(g^h_t,g_t^{h+1}; \zeta_s)$ (when we only evaluate the loss with only one trajectory, we directly use the notation $\zeta_s$). Taking expectation, we have
\begin{equation}
    \Eb_{\pi_s} L^h(g_t^h, g_t^{h+1}; \zeta_s) =  \left[\Eb_{\pi_s} \cE_h(g_t;x^h,a^h,b^h)\right]^2 + \sigma^2,
\end{equation}
where $\sigma^2$ is the expectation of $(\cT_h f^{h+1}(x_s^h, a_s^h, b_s^h) - r_s^h - V_{f^{h+1}}(x_s^{h+1}))^2$ or the variance, which is hard to deal with. However, the denominator in the likelihood allows us to rewrite the algorithm by replacing $L^h(g_t^h, g_t^{h+1}; \zeta_s)$ with the following excess loss:
\begin{align}
        &\Delta L^h(f^h, f^{h+1}; \zeta_s) := L^h(f^h,f^{h+1}; \zeta_s) \notag\\
        & \qquad - (\cT_h f^{h+1}(x_s^h, a_s^h, b_s^h) - r_s^h - V_{f^{h+1}}(x_s^{h+1}))^2,
\end{align}
whose expectation is the desired $\left[\Eb_{\pi_s} \cE_h(g_t;x^h,a^h,b^h)\right]^2$. After resolving the issue of variance, the analysis follows from the online aggregation techniques. However, the completeness assumption is required to analyze the introduced denominator (see the proof of Lemma~\ref{lem:lower_bound} and Lemma~\ref{lem:lower_bound2}).

\subsection{More Intuition}
We emphasize that the feature of the self-play setting that the learning agent can control both the max-player and the min-player plays a central role in the algorithm design and analysis. This allows us to decompose the regret into two parts as in Eqn.~\eqref{eqn:decom} and further employ the asymmetric structure to handle two parts. The analysis in the single-agent case essentially relies on the Markov property of transition, (conditional) sub-Gaussianity of the noise of transition, and the fact that the regret in one episode is upper bounded by the sum of Bellman residuals. We note that both the main agent and the booster agent retain these properties separately. Therefore, the techniques from MDP can be applied but with some additional efforts to handle the game nature.

%The main difference is that in this case we are concerning a different function class $\{\cE_h^{\mu_f}(g;\cdot,\cdot,\cdot): \cX \times \cA \times \cB \to \RR : (f,g,h) \in \cF \times \cF \times [H]\}$ and a different distribution family $\cD_\cF$. Therefore, by replacing the definitions of Bellman operators, Bellman residuals, and the decoupling coefficient accordingly, the techniques in \citet{dann2021provably} can be applied but with some additional efforts to handle the game nature.

\subsection{Complexity of $\cF$}
For optimization-based algorithms, the complexity of the function class $\cF$ is usually characterized through the cardinality $|\cF|$ or the covering number \cite{jiang@2017, wang2020reinforcement, jin2021bellman, jin2021power, huang2021towards}. On the other hand, the posterior sampling algorithm employs a prior distribution $p_0$ over $\cF$, which allows the algorithm to favor certain parts of it. Accordingly, our theoretical result depends on the complexity of $\cF$ through the prior preference, which is characterized by the following quantity.

\begin{definition} \label{def:quantity}%For any $f' \in \cF_{h+1}$, we define the set 
%$$\cF_h(\epsilon, f') := \{f \in \cF_h: \sup_{x,a,b} |\cE_h(f,f';x,a,b)| \leq \epsilon \}$$
%containing the functions that have small $\cT_h$-Bellman error against $f'$ for all state-action pairs. We then define the quantity
%$$\kappa_1(\epsilon) = \sup_{f\in \cF} \sum_{h=1}^H \ln \frac{1}{p_0^h(\cF_h(\epsilon, f^{h+1}))},$$
%which is the probability assigned by the prior to functions that approximately satisfy the Bellman equation w.r.t. $f$ for all state-action pair. Similarly, for a policy $\mu$ induced by some function in $\cF$ and for any function $f' \in \cF_{h+1}$, we define
For a policy $\mu_f, f \in \cF$ and for any function $g' \in \cF_{h+1}$, we define
\begin{equation*}\small
    \cF^{\mu_f}_h(\epsilon, g') = \{g \in \cF_h: \sup\nolimits_{x,a,b} |\cE^{\mu_f}_h(g,g';x,a,b)| \leq \epsilon \},
\end{equation*}
containing the functions that have small $\cT^{\mu_f}_h$-Bellman error against $g'$ for all state-action pairs. We then define
$$\kappa_\mu(\epsilon) = \sup\nolimits_{g\in \cF} \sum_{h=1}^H \ln \left(1/p_0^h(\cF^\mu_h(\epsilon, g^{h+1}))\right),$$
and
$$\kappa(\epsilon) = \sup\nolimits_{f \in \cF} \kappa_{\mu_f}(\epsilon).$$
\end{definition}
Under Assumption~\ref{assu:completeness}, it is assumed that
$\kappa(\epsilon) < \infty$, which is supported by the following two specific examples. 

For the finite function class with completeness, with a uniform prior $p_0^h(f) = 1/|\cF_h|$, we have
$$\kappa(\epsilon) \leq \sum_{h=1}^H \ln |\cF_h| = \ln |\cF|,$$
due to the realizability assumption. For an infinite function class, by replacing $|\cF|$ with its covering number, similar result can also be ontained.
%instead we can consider its covering and replace the cardinality $|\cF|$ with its covering number $\mathcal{N}_\infty(\cF, \epsilon)$. The results are still valid with a slightly modified proof.

For a $d$-dimensional parametric models $\cF_h = \{g_\theta \in \RR^d: \theta \in \Omega_h\}$ where $\Omega_h$ is compact, we can generally assume that 
$
\sup _{\theta} \ln \frac{1}{p_{0}^{h}\left(\left\{\theta^{\prime}:\left\|\theta^{\prime}-\theta\right\| \leq \epsilon\right\}\right)} \leq d \ln \left(c^{\prime} / \epsilon\right)
$
for some constant $c'$ depending on the prior. If we further assume that $g_\theta$ is Lipschitz in $\theta$ (e.g., linear MG \cite{xie2020learning}), then we can assume that $
\ln \frac{1}{p_{0}^{h}\left(\mathcal{F}_{h}^{\mu_f}\left(\epsilon, g^{h+1}\right)\right)} \leq c_{0} d \ln \left(c_{1} / \epsilon\right)
$
for some constants $c_0$ and $c_1$ depending on the prior and the Lipschitz constant $L$. In this case, we have
$$\kappa(\epsilon) \leq c_0Hd\ln(c_1/\epsilon).$$

\section{Main Results}\label{sec:results}
In this section, we state the main theoretical result of this paper and interpret it using several examples.
\subsection{Theoretical Guarantee}
We now provide an upper bound for the overall regret.
\begin{theorem}[Overall regret] \label{thm:thm}% Let Assumptions~\ref{assu:realizability}, \ref{assu:completeness} and~\ref{assu:bounded} hold. If ${\eta}{\beta^2} \leq 0.5$, and $\lambda \beta^2 \geq 1$ hold, and we further take $\epsilon = \frac{\beta}{T^2}$, $\lambda = \sqrt{\frac{T\kappa(\beta/T^2)}{\beta^2dc(\cF,MG,T)}}$, $\eta  = \frac{1}{4\beta^2}$,
Let Assumptions~\ref{assu:realizability}, \ref{assu:completeness} and~\ref{assu:bounded} hold. If ${\eta}{\beta^2} \leq 0.5$ and $\lambda \beta^2 \geq 1$ hold, and let $dc(\cF,MG,T)$ be an upper bound for the $\sup_{\mu \leq 1}dc(\cF,MG,T,\mu)$, and we further take $\lambda = \sqrt{\frac{T\kappa(\frac{\beta}{T^2})}{\beta^2dc(\cF,MG,T)}}$, $\eta  = \frac{1}{4\beta^2}$, then, it holds that
\begin{equation*}\small
\begin{aligned}
\Eb \Reg(T) \leq O\bigg(\beta\sqrt{dc(\cF,MG,T)\kappa(\frac{\beta}{T^2})T} + {dc(\cF,MG,T)}\bigg). 
\end{aligned}
\end{equation*}
\end{theorem}
Notably, if the multi-agent decoupling coefficient is provably small, Algorithm~\ref{alg:mg_TS} admits a $\sqrt{T}$-regret. According to the decomposition in Eqn.~\eqref{eqn:decom}, Theorem~\ref{thm:thm} can be established once we can bound the main agent part and the booster agent part.

\begin{theorem}[Bound of the main agent] \label{thm:adversarial}
With the same conditions as Theorem~\ref{thm:thm}, it holds that
\begin{equation*}\small
\begin{aligned}
&\sum_{t=1}^T\Eb_{S_{t-1}} \Eb_{f_t \sim p_t} \Eb_{g_t \sim p^{\mu_t}_t} \left[V_1^*(x^1) - V_1^{\mu_t, \nu_t}(x^1)\right] \\
&\leq O\bigg(\beta\sqrt{dc(\cF,MG,T)\kappa(\frac{\beta}{T^2})T} + dc(\cF,MG,T)\bigg). 
\end{aligned}
\end{equation*}
\end{theorem}
%Notably, theorem~\ref{thm:adversarial} states that we can achieve a $\sqrt{T}$-regret against the adversary even though we cannot control the behavior of the min-player, provided that the problem is of a finite decoupling coefficient. 
We then turn to the booster agent and provide an upper bound for the regret induced by approximating the best response policy. 

\begin{theorem}[Bound of the booster agent] \label{thm:booster agent}% Let Assumptions~\ref{assu:realizability}, \ref{assu:completeness} and~\ref{assu:bounded} hold. If ${\eta}{\beta^2} \leq 0.5$, and $\lambda \beta^2 \geq 1$ hold, and we further take $\epsilon = \frac{\beta}{T^2}$, $\lambda = \sqrt{\frac{T\kappa(\beta/T^2)}{\beta^2dc(\cF,MG,T)}}$, $\eta  = \frac{1}{4\beta^2}$, then, we have
With the same conditions as Theorem~\ref{thm:thm}, it holds that
\begin{equation*}\small
\begin{aligned}
&\sum_{t=1}^T\Eb_{S_{t-1}} \Eb_{f_t \sim p_t}\Eb_{g_t \sim p^{\mu_t}_t}  \left[V_1^{\mu_t, \nu_t}(x^1) - V_1^{\mu_t, \dagger}(x^1)\right] \\
&\leq O\bigg(\beta\sqrt{dc(\cF,MG,T)\kappa(\beta/T^2)T} + {dc(\cF,MG,T)}\bigg). 
\end{aligned}
\end{equation*}
\end{theorem}

%By the decomposition in \ref{eqn:decom}, Theorem~\ref{thm:thm} is proved once we can establish the upper bounds in Theorems~\ref{thm:adversarial} and~\ref{thm:booster agent}. 
The detailed proofs can be found in the appendix. %Instead, we provide several examples to interpret the theoretical result.

\subsection{Bounds for the Multi-Agent Decoupling Coefficient}
In this subsection, we provide several examples whose multi-agent decoupling coefficient is provably small. The proof can be found in Appendix~\ref{appen:dc_bounds}.% due to space constraint.

\textbf{Linear MG.} The first example is the MG with linear function approximation \cite{xie2020learning}. In this case, there exists a feature map $\phi(x,a,b) \in \RR^d$ and it holds that $r^h(x,a,b) = \phi(x,a,b)^\top \theta^h_*$ and $\Pb^h(x'|x,a,b) = \phi(x,a,b)^\top \mu_h(x')$ for some unknown $\theta^h_* \in \RR^d$ and $\mu_h(\cdot) \in \RR^d$ satisfying $\max\{\norm{\theta^h_*}, \norm{\mu_h}\} \leq \sqrt{d}$. We have the following upper bound for the multi-agent decoupling coefficient.
\begin{proposition}[Linear MG] \label{prop:linear} For a d-dimensional MG with $\cF_h = \{\phi_h(\cdot,\cdot,\cdot)^\top \theta^h: \norm{\theta^h} \leq (H+1-h)\sqrt{d}\}$ and $\norm{\phi(x,a,b)} \leq 1, \forall (x,a,b) \in \cX \times \cA \times \cB$, then for all $\mu \leq 1$, it holds that 
\begin{equation*}\small
    dc(\cF,MG,T,\mu) \leq 2dH(2+\ln(2HT)). 
\end{equation*}
\end{proposition}
Note that \citet{jin2021power} considers a more general setting of linear function approximation whose multi-agent decoupling coefficient is also provably small due to Proposition~\ref{prop:reduction_be}. Also note that as a special case, tabular MG is a linear MG of dimension $d=|\cX||\cA||\cB|$. 

%Therefore, we have
%\begin{proposition}[Tabular MG] For a tabular MG with state set whose state space $\cX$ and action spaces $\cA$ and $\cB$ are finite, if we consider a function class $\cF \subset (\cX \times \cA \times \cB \to [0,\beta-1])$, then for all $\mu \leq 1$, we have
%\begin{equation}
%dc(\cF,MG,T, \mu) \leq 2H|\cX||\cA||\cB|(2+\ln(2HT)).
%\end{equation}
%\end{proposition}

\textbf{Generalized Linear MG.} We then consider the generalized linear MG. In this case, we have $(f^h - \cT^\mu_h f_{h+1})(x,a,b) = \sigma(\phi(x,a,b)^\top \theta^h)$ for any $\mu$ induced by some function in $\cF$ and $f \in \cF$ where $\sigma$ is differentiable and strictly increasing. We further assume that $\sigma' \in (c_1, c_2)$ and $\max\{\norm{\phi(x,a,b)}, \norm{\theta^h}\}\leq R$ for some $c_1,c_2,R > 0$.
\begin{proposition}[Generalized Linear MG.] \label{prop:gen_linear}For a generalized linear MG, with $\cF=\{(x,a,b) \to \sigma(\phi(x,a,b)^\top \theta: \norm{\theta} \leq H\sqrt{d}\}$, then for all $\mu \leq 1$, it holds that  
\begin{equation*}
    dc(\cF, MG, T, \mu) \leq 2dH(c_2^2/c_1^2)(2+\ln(2HT)).
\end{equation*}
\end{proposition}

We can also derive an upper bound for the multi-agent decoupling coefficient through multi-agent Bellman Eluder dimension introduced in \citet{jin2021power}.
\begin{proposition}[Reduction to multi-agent Bellman Eluder dimension]
\label{prop:reduction_be}
Let $\Pi_{\cF} = \cD_\cF$ be the set of probability measures over $\cX \times \cA \times \cB$ at each step $h$ obtained by following $(\mu_f, \nu_{f,g})$ for some $f, g \in \cF$. If $\mathrm{dim_{BE}}(\cF, \Pi, 1/T) = E$, then the multi-agent decoupling coefficient satisfies:
$$dc(\cF, MG, T, \mu) \leq (1+\log(T) + 8\mu)EH.$$
\end{proposition}

Similar to the single-agent case, the multi-agent decoupling coefficient exhibits an additional factor of $H$ due to the formulation of summation over all steps instead of maximum as in the multi-agent Bellman Eluder dimension case. This formulation can offer advantages when the complexity of the function class varies with time steps $h$. Combining this with Theorem~\ref{thm:thm}, the regret bound of our algorithm matches that of OFU-based algorithms. However, we do remark that the results of \citet{jin2021power, huang2021towards} are in a high-probability fashion, which is stronger than the bound in expectation.

\subsection{Interpretation of Theorem~\ref{thm:thm}} \label{sec:illu}
We now illustrate Theorem~\ref{thm:thm} by concrete examples. The first example is for the finite function classes.

\begin{corollary}[Finite function classes with completeness] Let $\cF$ be a finite function class satisfying Assumptions~\ref{assu:realizability}, \ref{assu:completeness} and~\ref{assu:bounded} with $\beta=2$. Assume that the prior is uniform $p_0^h(f) = 1/|\cF_h|$, and $|\cF| = \prod_{h=1}^H |\cF_h|$. With $\eta = 0.1$ and $\lambda = \sqrt{\frac{T\ln|\cF|}{dc(\cF,MG,T)}}$, we have
\begin{equation*}\small
    \Eb \Reg(T) = O(\sqrt{dc(\cF, MG, T)T\ln(|\cF|}).
\end{equation*}
\end{corollary}
Note that it is straightforward to generalize this result to the infinite function classes by replacing the cardinality $|\cF|$ with its covering number $\mathcal{N}_\infty(\cF, \epsilon)$ with an appropriate choice of $\epsilon$. We then illustrate Theorem~\ref{thm:thm} by considering the MGs with linear function approximation. 
\begin{corollary}[Linear MG] For the linear MG, if we assume that the prior is uniform, we have $\kappa(\epsilon) = O(Hd\ln(1/\epsilon))$. With $\eta = \frac{0.4}{H^2}$ and $\lambda = \sqrt{\frac{T\kappa(H/T^2)}{dH^3(1+\ln(2HT))}}$, we have
\begin{equation*}
    \Eb \Reg(T) = {O}(H^2d\sqrt{T}\ln(HT)).
\end{equation*}
\end{corollary}
Compared with \citet{xie2020learning}, our algorithm improves the regret bound for linear MGs by a factor of $\sqrt{d}$. We remark that the improvement is mainly due to the \textit{global} optimism mechanism instead of a step-wise one. Specifically, we add an optimistic term only in the prior distributions at the initial value as in Eqn.~\eqref{eqn:prior} and Eqn.~\eqref{eqn:booster agent_prior}. On the contrary, OMVI from \citet{xie2020learning} establishes optimism at every step (see lines $8$ and $9$ of their pseudo code). The main bottleneck is that due to the temporal dependency, OMVI needs to construct uniform concentration for the optimistic bonus function at every step, whose covering number leads to the extra $\sqrt{d}$ factor. See Eqn.~5 and Lemma 18 of \citet{xie2020learning} for details.

Recently, \citet{xiong2022nearly} adopt the dataset splitting trick from MDP \citep{xie2021policy} to resolve this issue in the offline setting where the trajectories are independently collected by some behavior policy. However, their technique cannot apply directly in online setting as the policy used to collect new trajectory depends on the history. Also, we remark that while the OMVI is also computationally efficient, both our posterior sampling algorithm and GOLF of \citet{jin2021power,huang2021towards} are only information-theoretic. Therefore, it remains open whether we could close this gap by designing computationally efficient algorithm. 

%Compared to the recent works on MGs \citep{jin2021power, huang2021towards}, our posterior sampling algorithm achieves comparable regret bound to that of algorithms based on the OFU principle in the two-player zero-sum MG. 

\section{Conclusion}
In this paper, a self-play posterior sampling algorithm is proposed for two-player zero-sum Markov games with general function approximation, which is the first to the best of our knowledge. A new complexity measure,  {\it multi-agent decoupling coefficient}, is introduced to characterize the complexity of function class. Rigorous theoretical analysis showed that the proposed algorithm could achieve 
comparable regret bounds compared with other OFU-based algorithms for problems with low multi-agent decoupling coefficient, which extends the results in the single-agent RL.% and improves the regret bound in the linear MG setting.

As existing algorithms with general function approximation are computationally inefficient in general, one important direction for future works is to design computationally tractable algorithms for MGs (and MDPs). Another interesting open question is how to extend the posterior sampling algorithms for general-sum Markov games.

\section*{Acknowledgements}
WX and TZ acknowledge the funding supported by GRF 16201320 and the Hong Kong Ph.D. Fellowship. The CSs acknowledge the funding support by the US National Science Foundation under Grant ECCS- 2029978, ECCS-2033671, and CNS-2002902, and the Bloomberg Data Science Ph.D. Fellowship.

\bibliography{ref}

\begin{thebibliography}{40}
\providecommand{\natexlab}[1]{#1}
\providecommand{\url}[1]{\texttt{#1}}
\expandafter\ifx\csname urlstyle\endcsname\relax
  \providecommand{\doi}[1]{doi: #1}\else
  \providecommand{\doi}{doi: \begingroup \urlstyle{rm}\Url}\fi

\bibitem[Agrawal et~al.(2020)Agrawal, Chen, and Jiang]{agrawal2020improved}
Agrawal, P., Chen, J., and Jiang, N.
\newblock Improved worst-case regret bounds for randomized least-squares value
  iteration.
\newblock \emph{arXiv preprint arXiv:2010.12163}, 2020.

\bibitem[Bai \& Jin(2020)Bai and Jin]{bai2020provable}
Bai, Y. and Jin, C.
\newblock Provable self-play algorithms for competitive reinforcement learning.
\newblock In \emph{International Conference on Machine Learning}, pp.\
  551--560. PMLR, 2020.

\bibitem[Bai et~al.(2020)Bai, Jin, and Yu]{bai2020near}
Bai, Y., Jin, C., and Yu, T.
\newblock Near-optimal reinforcement learning with self-play.
\newblock \emph{arXiv preprint arXiv:2006.12007}, 2020.

\bibitem[Berner et~al.(2019)Berner, Brockman, Chan, Cheung, D{k{e}}biak,
  Dennison, Farhi, Fischer, Hashme, Hesse, et~al.]{berner2019dota}
Berner, C., Brockman, G., Chan, B., Cheung, V., D{k{e}}biak, P., Dennison, C.,
  Farhi, D., Fischer, Q., Hashme, S., Hesse, C., et~al.
\newblock Dota 2 with large scale deep reinforcement learning.
\newblock \emph{arXiv preprint arXiv:1912.06680}, 2019.

\bibitem[Brown \& Sandholm(2019)Brown and Sandholm]{brown2019superhuman}
Brown, N. and Sandholm, T.
\newblock Superhuman ai for multiplayer poker.
\newblock \emph{Science}, 365\penalty0 (6456):\penalty0 885--890, 2019.

\bibitem[Chapelle \& Li(2011)Chapelle and Li]{chapelle2011empirical}
Chapelle, O. and Li, L.
\newblock An empirical evaluation of thompson sampling.
\newblock \emph{Advances in neural information processing systems}, 24, 2011.

\bibitem[Chen et~al.(2021)Chen, Zhou, and Gu]{chen2021almost}
Chen, Z., Zhou, D., and Gu, Q.
\newblock Almost optimal algorithms for two-player {M}arkov games with linear
  function approximation.
\newblock \emph{arXiv preprint arXiv:2102.07404}, 2021.

\bibitem[Dann et~al.(2021)Dann, Mohri, Zhang, and Zimmert]{dann2021provably}
Dann, C., Mohri, M., Zhang, T., and Zimmert, J.
\newblock A provably efficient model-free posterior sampling method for
  episodic reinforcement learning.
\newblock \emph{Advances in Neural Information Processing Systems}, 34, 2021.

\bibitem[Du et~al.(2021)Du, Kakade, Lee, Lovett, Mahajan, Sun, and
  Wang]{du2021bilinear}
Du, S.~S., Kakade, S.~M., Lee, J.~D., Lovett, S., Mahajan, G., Sun, W., and
  Wang, R.
\newblock Bilinear classes: A structural framework for provable generalization
  in rl.
\newblock \emph{arXiv preprint arXiv:2103.10897}, 2021.

\bibitem[Filar \& Vrieze(2012)Filar and Vrieze]{filar2012competitive}
Filar, J. and Vrieze, K.
\newblock \emph{Competitive Markov decision processes}.
\newblock Springer Science \& Business Media, 2012.

\bibitem[Huang et~al.(2021)Huang, Lee, Wang, and Yang]{huang2021towards}
Huang, B., Lee, J.~D., Wang, Z., and Yang, Z.
\newblock Towards general function approximation in zero-sum markov games.
\newblock \emph{arXiv preprint arXiv:2107.14702}, 2021.

\bibitem[Jafarnia-Jahromi et~al.(2021)Jafarnia-Jahromi, Jain, and
  Nayyar]{jafarnia2021learning}
Jafarnia-Jahromi, M., Jain, R., and Nayyar, A.
\newblock Learning zero-sum stochastic games with posterior sampling.
\newblock \emph{arXiv preprint arXiv:2109.03396}, 2021.

\bibitem[Jiang et~al.(2017)Jiang, Krishnamurthy, Agarwal, Langford, and
  Schapire]{jiang@2017}
Jiang, N., Krishnamurthy, A., Agarwal, A., Langford, J., and Schapire, R.~E.
\newblock Contextual decision processes with low {B}ellman rank are
  {PAC}-learnable.
\newblock In \emph{Proceedings of the 34th International Conference on Machine
  Learning}, volume~70 of \emph{Proceedings of Machine Learning Research}, pp.\
   1704--1713. PMLR, 06--11 Aug 2017.

\bibitem[Jin et~al.(2021{\natexlab{a}})Jin, Liu, and
  Miryoosefi]{jin2021bellman}
Jin, C., Liu, Q., and Miryoosefi, S.
\newblock Bellman eluder dimension: New rich classes of rl problems, and
  sample-efficient algorithms.
\newblock \emph{arXiv preprint arXiv:2102.00815}, 2021{\natexlab{a}}.

\bibitem[Jin et~al.(2021{\natexlab{b}})Jin, Liu, and Yu]{jin2021power}
Jin, C., Liu, Q., and Yu, T.
\newblock The power of exploiter: Provable multi-agent rl in large state
  spaces.
\newblock \emph{arXiv preprint arXiv:2106.03352}, 2021{\natexlab{b}}.

\bibitem[Kaufmann et~al.(2012)Kaufmann, Korda, and Munos]{kaufmann2012thompson}
Kaufmann, E., Korda, N., and Munos, R.
\newblock Thompson sampling: An asymptotically optimal finite-time analysis.
\newblock In \emph{International conference on algorithmic learning theory},
  pp.\  199--213. Springer, 2012.

\bibitem[Krishnamurthy et~al.(2016)Krishnamurthy, Agarwal, and
  Langford]{krishnamurthy2016pac}
Krishnamurthy, A., Agarwal, A., and Langford, J.
\newblock Pac reinforcement learning with rich observations.
\newblock \emph{arXiv preprint arXiv:1602.02722}, 2016.

\bibitem[Littman(1994)]{littman1994markov}
Littman, M.~L.
\newblock Markov games as a framework for multi-agent reinforcement learning.
\newblock In \emph{Machine learning proceedings 1994}, pp.\  157--163.
  Elsevier, 1994.

\bibitem[Liu et~al.(2020)Liu, Yu, Bai, and Jin]{liu2020sharp}
Liu, Q., Yu, T., Bai, Y., and Jin, C.
\newblock A sharp analysis of model-based reinforcement learning with
  self-play.
\newblock \emph{arXiv preprint arXiv:2010.01604}, 2020.

\bibitem[Osband \& Van~Roy(2014)Osband and Van~Roy]{osband2014model}
Osband, I. and Van~Roy, B.
\newblock Model-based reinforcement learning and the eluder dimension.
\newblock \emph{arXiv preprint arXiv:1406.1853}, 2014.

\bibitem[Osband et~al.(2016)Osband, Van~Roy, and Wen]{osband2016generalization}
Osband, I., Van~Roy, B., and Wen, Z.
\newblock Generalization and exploration via randomized value functions.
\newblock In \emph{International Conference on Machine Learning}, pp.\
  2377--2386. PMLR, 2016.

\bibitem[Perolat et~al.(2015)Perolat, Scherrer, Piot, and
  Pietquin]{perolat2015approximate}
Perolat, J., Scherrer, B., Piot, B., and Pietquin, O.
\newblock Approximate dynamic programming for two-player zero-sum markov games.
\newblock In \emph{International Conference on Machine Learning}, pp.\
  1321--1329. PMLR, 2015.

\bibitem[Russo(2019)]{russo2019worst}
Russo, D.
\newblock Worst-case regret bounds for exploration via randomized value
  functions.
\newblock \emph{arXiv preprint arXiv:1906.02870}, 2019.

\bibitem[Russo \& Van~Roy(2014)Russo and Van~Roy]{russo2014learning}
Russo, D. and Van~Roy, B.
\newblock Learning to optimize via posterior sampling.
\newblock \emph{Mathematics of Operations Research}, 39\penalty0 (4):\penalty0
  1221--1243, 2014.

\bibitem[Shalev-Shwartz et~al.(2016)Shalev-Shwartz, Shammah, and
  Shashua]{shalev2016safe}
Shalev-Shwartz, S., Shammah, S., and Shashua, A.
\newblock Safe, multi-agent, reinforcement learning for autonomous driving.
\newblock \emph{arXiv preprint arXiv:1610.03295}, 2016.

\bibitem[Shapley(1953)]{shapley1953stochastic}
Shapley, L.~S.
\newblock Stochastic games.
\newblock \emph{Proceedings of the national academy of sciences}, 39\penalty0
  (10):\penalty0 1095--1100, 1953.

\bibitem[Silver et~al.(2016)Silver, Huang, Maddison, Guez, Sifre, Van
  Den~Driessche, Schrittwieser, Antonoglou, Panneershelvam, Lanctot,
  et~al.]{silver2016mastering}
Silver, D., Huang, A., Maddison, C.~J., Guez, A., Sifre, L., Van Den~Driessche,
  G., Schrittwieser, J., Antonoglou, I., Panneershelvam, V., Lanctot, M.,
  et~al.
\newblock Mastering the game of go with deep neural networks and tree search.
\newblock \emph{nature}, 529\penalty0 (7587):\penalty0 484--489, 2016.

\bibitem[Silver et~al.(2017)Silver, Schrittwieser, Simonyan, Antonoglou, Huang,
  Guez, Hubert, Baker, Lai, Bolton, et~al.]{silver2017mastering}
Silver, D., Schrittwieser, J., Simonyan, K., Antonoglou, I., Huang, A., Guez,
  A., Hubert, T., Baker, L., Lai, M., Bolton, A., et~al.
\newblock Mastering the game of go without human knowledge.
\newblock \emph{nature}, 550\penalty0 (7676):\penalty0 354--359, 2017.

\bibitem[Sun et~al.(2019)Sun, Jiang, Krishnamurthy, Agarwal, and
  Langford]{sun2019model}
Sun, W., Jiang, N., Krishnamurthy, A., Agarwal, A., and Langford, J.
\newblock Model-based rl in contextual decision processes: Pac bounds and
  exponential improvements over model-free approaches.
\newblock In \emph{Conference on learning theory}, pp.\  2898--2933. PMLR,
  2019.

\bibitem[Van~Handel(2014)]{van2014probability}
Van~Handel, R.
\newblock Probability in high dimension.
\newblock Technical report, PRINCETON UNIV NJ, 2014.

\bibitem[Wang et~al.(2020)Wang, Salakhutdinov, and Yang]{wang2020reinforcement}
Wang, R., Salakhutdinov, R., and Yang, L.~F.
\newblock Reinforcement learning with general value function approximation:
  Provably efficient approach via bounded eluder dimension.
\newblock \emph{arXiv preprint arXiv:2005.10804}, 2020.

\bibitem[Weisz et~al.(2021)Weisz, Amortila, and
  Szepesv{\'a}ri]{weisz2021exponential}
Weisz, G., Amortila, P., and Szepesv{\'a}ri, C.
\newblock Exponential lower bounds for planning in mdps with
  linearly-realizable optimal action-value functions.
\newblock In \emph{Algorithmic Learning Theory}, pp.\  1237--1264. PMLR, 2021.

\bibitem[Xie et~al.(2020)Xie, Chen, Wang, and Yang]{xie2020learning}
Xie, Q., Chen, Y., Wang, Z., and Yang, Z.
\newblock Learning zero-sum simultaneous-move {M}arkov games using function
  approximation and correlated equilibrium.
\newblock In \emph{Conference on Learning Theory}, pp.\  3674--3682. PMLR,
  2020.

\bibitem[Xie et~al.(2021)Xie, Jiang, Wang, Xiong, and Bai]{xie2021policy}
Xie, T., Jiang, N., Wang, H., Xiong, C., and Bai, Y.
\newblock Policy finetuning: Bridging sample-efficient offline and online
  reinforcement learning.
\newblock \emph{Advances in neural information processing systems}, 34, 2021.

\bibitem[Xiong et~al.(2022)Xiong, Zhong, Shi, Shen, Wang, and
  Zhang]{xiong2022nearly}
Xiong, W., Zhong, H., Shi, C., Shen, C., Wang, L., and Zhang, T.
\newblock Nearly minimax optimal offline reinforcement learning with linear
  function approximation: Single-agent mdp and markov game.
\newblock \emph{arXiv preprint arXiv:2205.15512}, 2022.

\bibitem[Xiong et~al.(2021)Xiong, Shen, and Du]{xiong2021randomized}
Xiong, Z., Shen, R., and Du, S.~S.
\newblock Randomized exploration is near-optimal for tabular mdp.
\newblock \emph{arXiv preprint arXiv:2102.09703}, 2021.

\bibitem[Zanette et~al.(2020)Zanette, Brandfonbrener, Brunskill, Pirotta, and
  Lazaric]{zanette2020frequentist}
Zanette, A., Brandfonbrener, D., Brunskill, E., Pirotta, M., and Lazaric, A.
\newblock Frequentist regret bounds for randomized least-squares value
  iteration.
\newblock In \emph{International Conference on Artificial Intelligence and
  Statistics}, pp.\  1954--1964. PMLR, 2020.

\bibitem[Zhang et~al.(2021)Zhang, Yang, and Ba{\c{s}}ar]{zhang2021multi}
Zhang, K., Yang, Z., and Ba{\c{s}}ar, T.
\newblock Multi-agent reinforcement learning: A selective overview of theories
  and algorithms.
\newblock \emph{Handbook of Reinforcement Learning and Control}, pp.\
  321--384, 2021.

\bibitem[Zhang(2005)]{zhang2005data}
Zhang, T.
\newblock Data dependent concentration bounds for sequential prediction
  algorithms.
\newblock In \emph{International Conference on Computational Learning Theory},
  pp.\  173--187. Springer, 2005.

\bibitem[Zhang(2021)]{zhang2021feel}
Zhang, T.
\newblock Feel-good thompson sampling for contextual bandits and reinforcement
  learning.
\newblock \emph{arXiv preprint arXiv:2110.00871}, 2021.

\end{thebibliography}
\bibliographystyle{icml2022}

%%%%%%%%%%%%%%%%%%%%%%%%%%%%%%%%%%%%%%%%%%%%%%%%%%%%%%%%%%%%%%%%%%%%%%%%%%%%%%%
%%%%%%%%%%%%%%%%%%%%%%%%%%%%%%%%%%%%%%%%%%%%%%%%%%%%%%%%%%%%%%%%%%%%%%%%%%%%%%%
% APPENDIX
%%%%%%%%%%%%%%%%%%%%%%%%%%%%%%%%%%%%%%%%%%%%%%%%%%%%%%%%%%%%%%%%%%%%%%%%%%%%%%%
%%%%%%%%%%%%%%%%%%%%%%%%%%%%%%%%%%%%%%%%%%%%%%%%%%%%%%%%%%%%%%%%%%%%%%%%%%%%%%%
\newpage
\appendix
\onecolumn

%\section{Related Works} \label{appen:related_work}
\section{Equivalent Algorithms}
We will consider a slightly more general posterior sampling algorithm with an extra parameter $\alpha \in (0, 1]$. We recall that the posterior defined in Eqn.~\eqref{eqn:posterior_main} is
$${p}(f|S_t) \propto \exp(\lambda V_{f,1}(x^1)) \prod_{h=1}^H q(f^h|f^{h+1}, S_t),$$ 
where 
$$
    q(f^h|f^{h+1}, S_t) = \frac{p_0^h(f^h)\exp
    \left(- \eta L^h(f^h, f^{h+1}; S_t)\right)}{\Eb_{f^h \sim p_0^h} \exp(-\eta L^h(f^h, f^{h+1}; S_t))}.
$$
Equivalently, we may consider the excess loss
\begin{align}
    \Delta L^h(f^h, f^{h+1}; \zeta_s) = &(f^h(x_s^h,a^h_s,b^h_s) - r_s^h - V_{f^{h+1}}(x^{h+1}_s))^2 \notag\\
    &\qquad- (\cT_h f^{h+1}(x_s^h, a_s^h, b_s^h) - r_s^h - V_{f^{h+1}}(x_s^{h+1}))^2,
\end{align}
where we employ the notation that $\zeta_s = \{[x_s^h, a_s^h, b_s^h, r_s^h]\}_{h=1}^H$. We then define the potential function as
\begin{equation}
\begin{aligned}
\Phi_{t}^{h}(f)=&-\ln  p_{0}^{h}\left(f^{h}\right)+\alpha \eta \sum_{s=1}^{t-1} \Delta L^{h}\left(f^{h}, f^{h+1} ; \zeta_{s}\right) \\
&\qquad+\alpha \ln \mathbb{E}_{\tilde{f}^{h} \sim p_{0}^{h}} \exp \left(-\eta \sum_{s=1}^{t-1} \Delta L^{h}\left(\tilde{f}^{h}, f^{h+1} ; \zeta_{s}\right)\right),
\end{aligned}
\end{equation}
where $\alpha \in (0,1]$ is the extra parameter to facilitate the proof. We also define
\begin{equation*}
    \Delta f^1(x^1) = V_{f,1}(x^1) - V_1^*(x^1).
\end{equation*}
Then, we obtain a generalized posterior distribution on $\cF$:
\begin{equation}
    \hat{p}_t(f) \propto \exp\left(-\sum_{h=1}^H \Phi_t^h(f) + \lambda \Delta f^1(x^1)\right),
\end{equation}
where it is equivalent to the posterior given in Eqn.~\eqref{eqn:posterior_main} when $\alpha = 1$.

We then recall the posterior distribution of the booster agent defined in Eqn.~\eqref{eqn:posterior_booster agent} is given by
$$
    p^\mu(g|S_t) \propto \exp(-\lambda V^\mu_{g, 1}(x^1)) \prod_{h=1}^H q^\mu(g^h|g^{h+1}, S_t),
$$
where 
$$
    q^\mu(g^h|g^{h+1}, S_t) = \frac{p_0^h(g^h)\exp
    \left(- \eta L^h_\mu(g^h, g^{h+1}; S_t)\right)}{\Eb_{g^h \sim p_0^h} \exp(-\eta L^h_\mu(g^h, g^{h+1}; S_t))}.
$$
Similarly, we define the excess loss for the booster agent:
\begin{equation}
\begin{aligned}
    \Delta L^h_\mu(g^h, g^{h+1}; \zeta_s) =& (g^h(x_s^h,a^h_s,b^h_s) - r_s^h - V^\mu_{g^{h+1}}(x^{h+1}_s))^2\\
    &\qquad- (\cT_h^\mu g^{h+1}(x_s^h, a_s^h, b_s^h) - r_s^h - V^\mu_{g^{h+1}}(x_s^{h+1}))^2.
\end{aligned}
\end{equation}
and
$$\Delta g^1_\mu(x^1) = V_1^{\mu,\dagger}(v^1) - V^\mu_{g,1}(x^1),$$
and use the following notation (with slight abuse of notation) for the potential function:
\begin{equation}
\begin{aligned}
\Phi_{t}^{h}(g, \mu)=&-\ln  p_{0}^{h}\left(g^{h}\right)+\alpha \eta \sum_{s=1}^{t-1} \Delta L^{h}_\mu\left(g^{h}, g^{h+1} ; \zeta_{s}\right) \\
&\qquad+\alpha \ln \mathbb{E}_{\tilde{g}^{h} \sim p_{0}^{h}} \exp \left(-\eta \sum_{s=1}^{t-1} \Delta L^{h}_\mu\left(\tilde{g}^{h}, g^{h+1} ; \zeta_{s}\right)\right),
\end{aligned}
\end{equation}
since the analyses for Algorithm~\ref{alg:main} and Algorithm~\ref{alg:booster agent} are separate so the meaning of $\Phi_t^h(\cdot)$ will be clear from the context.
Finally, we obtain a generalized posterior function for the booster agent: 
\begin{equation}
    \hat{p}^\mu_t(g) \propto \exp\left(-\sum_{h=1}^H \Phi_t^h(g, \mu) + \lambda \Delta g^1_\mu(x^1)\right).
\end{equation}

The main motivation to use $\Delta L^h(\cdot)$ ($\Delta L^h_\mu(\cdot)$) is that the variance will be cancelled during our theoretical analysis as it is equivalent to the case where we know the Bellman operator. This is possible because the novel denominator term is introduced in the likelihood as in \citet{dann2021provably}.

\section{Useful Lemmas and Additional Notations}
In this section, we provide several useful lemmas and additional notations that are useful later. We start with the following definitions. First, we further define a quantity similar to Definition~\ref{def:quantity}, which will be used for the analysis of the main agent.
\begin{definition}
For any $f' \in \cF_{h+1}$, we define the set 
$$\cF_h(\epsilon, f') := \{f \in \cF_h: \sup_{x,a,b} |\cE_h(f,f';x,a,b)| \leq \epsilon \}$$
containing the functions that have small $\cT_h$-Bellman error against $f'$ for all state-action pairs. We then define the quantity
$$\kappa_1(\epsilon) = \sup_{f\in \cF} \sum_{h=1}^H \ln \frac{1}{p_0^h(\cF_h(\epsilon, f^{h+1}))},$$
which is the probability assigned by the prior to functions that approximately satisfy the Bellman equation w.r.t. $f$ for all state-action pair. 
\end{definition}
Note that $\kappa_1(\epsilon) \leq \kappa(\epsilon)$ because $$\kappa_1(\epsilon) = \sup_{g\in \cF} \sum_{h=1}^H \ln \frac{1}{p_0^h(\cF^{\mu_g}_h(\epsilon, g^{h+1}))} \leq \sup_{f \in \cF}\sup_{g\in \cF} \sum_{h=1}^H \ln \frac{1}{p_0^h(\cF^{\mu_f}_h(\epsilon, g^{h+1}))} = \kappa(\epsilon).$$

\begin{definition}
For $\alpha \in (0,1)$, we also use the notations:
$$
\begin{aligned}
\kappa^h_1(\alpha, \epsilon) =  (1-\alpha) \ln \Eb_{f^{h+1} \sim p_0^{h+1}} p_0^h(\cF_h(\epsilon, f^{h+1}))^{-\alpha/(1-\alpha)},
\end{aligned}
$$
and $\kappa^h_1(1,\epsilon) = \lim_{\alpha \to 1^{-}} \kappa^h(\alpha, \epsilon)$ where it holds that
$$
\kappa^{h}_1(1, \epsilon)=\sup _{f^{h+1} \in \mathcal{F}_{h+1}} \ln \frac{1}{p_{0}^{h}\left(\mathcal{F}_{h}\left(\epsilon, f^{h+1}\right)\right)}<\infty,
$$
and 
$$\kappa_1(\epsilon) = \sum_{h=1}^H \kappa_1^h(1,\epsilon) \leq \kappa(\epsilon).$$
Similarly, we define
$$\kappa^h_\mu(\alpha, \epsilon) = (1-\alpha) \ln \Eb_{f^{h+1} \sim p_0^{h+1}} p_0^h(\cF^\mu_h(\epsilon, f^{h+1}))^{-\alpha/(1-\alpha)},$$ 
and $\kappa^h_\mu(1,\epsilon) = \lim_{\alpha \to 1^{-}} \kappa^h_\mu(\alpha, \epsilon)$, 
Then, it holds that 
$$\kappa_\mu^h(1,\epsilon) = \sup_{f^{h+1} \in \cF_{h+1}} \ln \frac{1}{p_0^h(\cF_h^\mu(\epsilon,f^{h+1}))} < \infty,$$
and
$$\kappa_\mu(\epsilon) = \sum_{h=1}^H \kappa^h_{\mu}(1,\epsilon) \leq \kappa(\epsilon).$$
\end{definition}

\begin{lemma}
\label{lem:martingale}
For any fixed $g \in \cF$ and max-player's policy $\mu:=\mu_f$ for some $f \in \cF$, we define a random variable for all $s$ and $h$ as follows:
$$
\begin{aligned}
\xi_{s}^{h}\left(g^{h}, g^{h+1}, \zeta_{s}\right)=&-2 \eta \Delta L^{h}_\mu \left(g^{h}, g^{h+1}, \zeta_{s}\right) -\ln \mathbb{E}_{x_s^{h+1} \sim \Pb^{h}\left(\cdot \mid x_{s}^{h}, a_{s}^{h}\right)} \exp \left(-2 \eta \Delta L^{h}_\mu\left(g^{h}, g^{h+1}, \zeta_{s}\right)\right).
\end{aligned}
$$
Then, for all $h$, we have
$$
\mathbb{E}_{S_{t-1}} \exp \left(\sum_{s=1}^{t-1} \xi_{s}^{h}\left(g^{h}, g^{h+1}, \zeta_{s}\right)\right)=1.
$$
A special case is that $f=g$ where we have
$$\Delta L_{\mu_f}^h(f^h, f^{h+1}, \zeta_s) = \Delta L^h(f^h, f^{h+1}, \zeta_s).$$
\end{lemma}
\begin{proof}
This lemma is from \citet{zhang2005data} and is also proved in \citet{dann2021provably}.
\end{proof}

\begin{lemma}
\label{lem: Gibbs_var}
Let $\nu$ be a probability distribution. Then, $\Eb_\nu f - H(\nu)$ is minimized at $\nu(x) \propto \exp(-f(x))$.
\end{lemma}
\begin{proof}
This is a corollary of Gibbs variational principle whose proof can be found in \citet{van2014probability}, Lemma 4.10.
\end{proof}
Using Lemma~\ref{lem: Gibbs_var}, we can obtain the following key lemma as used in \citet{dann2021provably}. 
\begin{lemma} \label{lem:key}
It holds that
\begin{equation}
\label{eqn:key}
\begin{aligned}
\mathbb{E}_{f \sim \hat{p}_{t}}\left(\sum_{h=1}^{H} \Phi_{t}^{h}(f)-\lambda \Delta f^{1}\left(x^{1}\right)+\ln \hat{p}_{t}(f)\right)&=\inf _{p} \mathbb{E}_{f \sim p(\cdot)}\left(\sum_{h=1}^{H} \Phi_{t}^{h}(f)-\lambda \Delta f^{1}\left(x^{1}\right)+\ln p(f)\right);\\
\mathbb{E}_{g \sim \hat{p}_{t}^\mu}\left(\sum_{h=1}^{H} \Phi_{t}^{h}(g, \mu)-\lambda \Delta g^{1}_\mu\left(x^{1}\right)+\ln \hat{p}^\mu_{t}(g)\right)&=\inf _{p} \mathbb{E}_{g \sim p(\cdot)}\left(\sum_{h=1}^{H} \Phi_{t}^{h}(g, \mu)-\lambda \Delta g^{1}_\mu\left(x^{1}\right)+\ln p(g)\right),
\end{aligned}
\end{equation}
where we remark that the definitions of $\Phi_t^h(\cdot)$ in two equations are different.
\end{lemma}
In what follows, we derive a lower bound of LHS of Eqn.~\eqref{eqn:key}, and an upper bound of RHS of Eqn.~\eqref{eqn:key} for the proof of Theorems~\ref{thm:adversarial} and~\ref{thm:booster agent}.

\section{Proof of the Theorem~\ref{thm:adversarial}}
In this section, we provide the proof for Theorem~\ref{thm:adversarial}. The proof provided in this section basically follows the same line of that of single-agent RL because essentially the algorithms employ the same properties of the problem as discussed in Section~\ref{sec:main_idea} and for the main agent, and the Bellman residuals is free of the min-player's policy. 

\begin{lemma} \label{lem:moment} For all functions $f\in \cF$, we have
$$
\mathbb{E}_{x_s^{h+1} \sim \Pb^{h}\left(\cdot \mid x_{s}^{h}, a_{s}^{h},b_s^h\right)} \Delta L^{h}\left(f^{h}, f^{h+1}, \zeta_{s}\right)=\left(\mathcal{E}_{h}\left(f ; x_{s}^{h}, a_{s}^{h}, b_s^h\right)\right)^{2}
$$
and
$$
\mathbb{E}_{x_s^{h+1} \sim \Pb^{h}\left(\cdot \mid x_{s}^{h}, a_{s}^{h}, b_s^h\right)} \Delta L^{h}\left(f^{h}, f^{h+1}, \zeta_{s}\right)^{2} \leq \frac{4 \beta^{2}}{3}\left(\mathcal{E}_{h}\left(f ; x_{s}^{h}, a_{s}^{h}, b_s^h\right)\right)^{2}
$$
\end{lemma}
\begin{proof}
We define the random variable
$$Z = f^h(x_s^h, a_s^h, b_s^h) - r_s^h - V_{f, {h+1}}(x_s^{h+1}).$$
Let $\Eb$ be conditioned on $[x_s^h, a_s^h, b_s^h]$. Then, the randomness is from the state transition and we have
$$\Eb Z = \cE_h(f; x_s^h, a_s^h, b_s^h).$$
We also have
$$\Delta L^h(f^h,f^{h+1}, \zeta_s) = Z^2 - (Z-\Eb Z)^2.$$
and 
$$\Eb [Z^2 - (Z-\Eb Z)^2] = \Eb Z^2 - \text{var}(Z) = (\Eb Z)^2 = (\cE_h(f;x_s^h,a_s^h,b_s^h))^2.$$
Also note that $Z \in [-\beta,\beta-1]$ and $\max Z - \min Z \leq \beta$ if it is conditioned on $[x_s^h,a_s^h,b_s^h]$, this implies that
$$\Eb (Z^2 - (Z-\Eb Z)^2)^2 \leq \frac{4}{3} \beta^2 (\Eb Z)^2.$$
\end{proof}

\begin{lemma} \label{lem:ln_expectation}
If the learning rate $\eta$ is sufficiently small such that $\eta \beta^2 \leq 0.8$, then for all functions $f \in \cF$, we have
$$
\begin{aligned}
&\ln \mathbb{E}_{x_s^{h+1} \sim \Pb^{h}\left(\cdot \mid x_{s}^{h}, a_{s}^{h}, b_s^h\right)} \exp \left(-\eta \Delta L^{h}\left(f^{h}, f^{h+1}, \zeta_{s}\right)\right) \\
&\qquad\leq \mathbb{E}_{x_s^{h+1} \sim \Pb^{h}\left(\cdot \mid x_{s}^{h}, a_{s}^{h}, b_s^h\right)} \exp \left(-\eta \Delta L^{h}\left(f^{h}, f^{h+1}, \zeta_{s}\right)\right)-1 \\
&\qquad\leq -0.25 \eta \left(\mathcal{E}_{h}\left(f ; x_{s}^{h}, a_{s}^{h}, b_s^h\right)\right)^{2}.
\end{aligned}
$$
\end{lemma}
\begin{proof}
With $\eta \beta^2 \leq 0.8$, for all $f \in \cF$, we have
$$-\eta \Delta L^h(f^h, f^{h+1}, \zeta_s) \leq 0.8.$$
This implies that 
$$
\begin{aligned}
&\exp\left(-\eta \Delta L^h(f^h, f^{h+1}, \zeta_s)\right) \\
&\leq 1 - \eta \Delta L^h(f^h, f^{h+1}, \zeta_s) + 0.67\eta^2 \Delta L^h(f^h, f^{h+1}, \zeta_s)^2,
\end{aligned}
$$
where we use the fact that $\psi(z) = (e^z-1-z)/z^2$ is increasing in $z$ and $\psi(0.8) < 0.67$. Therefore, we have
$$
\begin{aligned}
&\ln \mathbb{E}_{x_s^{h+1} \sim \Pb^h\left(\cdot \mid x_{s}^{h}, a_{s}^{h}, b_s^h\right)} \exp \left(-\eta \Delta L^{h}\left(f^{h}, f^{h+1}, \zeta_{s}\right)\right) \\
&\qquad\leq \mathbb{E}_{x_s^{h+1} \sim \Pb^h\left(\cdot \mid x_{s}^{h}, a_{s}^{h}, b_s^h\right)} \exp \left(-\eta \Delta L^{h}\left(f^{h}, f^{h+1}, \zeta_{s}\right)\right) - 1 \\
&\qquad\leq \mathbb{E}_{x_s^{h+1} \sim \Pb^h\left(\cdot \mid x_{s}^{h}, a_{s}^{h}, b_s^h\right)} - \eta \Delta L^h(f^h, f^{h+1}, \zeta_s) + 0.67\eta^2 \Delta L^h(f^h, f^{h+1}, \zeta_s)^2 \\
&\qquad\leq -0.25 \eta (\cE_h(f^h,f^{h+1}, \zeta_s)^2,
\end{aligned}
$$
where the first inequality is because $\ln z \leq z - 1$ and the last inequality is because Lemma~\ref{lem:moment} and $(\frac{4}{3}\eta b^2 0.67) \leq 0.75$.
\end{proof}

\begin{lemma} \label{lem:upper}
It holds that 
$$
\begin{aligned}
\inf _{p} \mathbb{E}_{S_{t-1}} \mathbb{E}_{f \sim p(\cdot)}\left[\sum_{h=1}^{H} {\Phi}_{t}^{h}(f)-\lambda \Delta f^{1}\left(x^{1}\right)+\ln p(f)\right] \leq \lambda \epsilon+4 \alpha \eta(t-1) H \epsilon^{2}-\sum_{h=1}^{H} \ln p_{0}^{h}\left(\mathcal{F}_{h}\left(\epsilon, Q_{h+1}^{*}\right)\right).
\end{aligned}
$$
\end{lemma}
\begin{proof}
Consider any fixed $f \in \cF$. For any $\tilde{f}^h \in \cF^h$ that depends on $S_{s-1}$, we obtain from Lemma~\ref{lem:ln_expectation} that
$$
\mathbb{E}_{\zeta_{s}} \exp \left(-\eta \Delta L^{h}\left(\tilde{f}^{h}, f^{h+1}, \zeta_{s}\right)\right)-1 \leq-0.25 \eta \mathbb{E}_{\zeta_{s}}\left(\tilde{f}^{h}(x, a)-\mathcal{T}_{h} f^{h+1}(x, a, b)\right)^{2} \leq 0.
$$
We let 
$$W_t^h := \Eb_{S_t} \Eb_{f \sim p(\cdot)} \ln \Eb_{\tilde{f}^h \sim p_0^h} \exp\left(-\eta \sum_{s=1}^t \Delta L^h(\tilde{f}^h, f^{h+1}, \zeta_s)\right),$$
and recall that 
$$
\hat{q}_{t}^{h}\left(\tilde{f}^{h} \mid f^{h+1}, S_{t-1}\right)=\frac{p_0^h(\tilde{f}^h)\exp \left(-\eta \sum_{s=1}^{t-1} \Delta L^{h}\left(\tilde{f}^{h}, f^{h+1}, \zeta_{s}\right)\right)}{\mathbb{E}_{\tilde{f}' \sim p_{0}^{h}} \exp \left(-\eta \sum_{s=1}^{t-1} \Delta L^{h}\left(\tilde{f}^{\prime}, f^{h+1}, \zeta_{s}\right)\right)}.
$$
We have 
\begin{align*}
W_{s}^{h}-W_{s-1}^{h} &= \mathbb{E}_{S_{s}} \mathbb{E}_{f \sim p(\cdot)} \ln \Eb_{\tilde{f}^h \sim p_0^h} \frac{\exp \left(-\eta \sum_{t=1}^{s-1} \Delta L^{h}\left(\tilde{f}^{h}, f^{h+1}, \zeta_{t}\right)\right)}{\mathbb{E}_{\tilde{f}' \sim p_{0}^{h}} \exp \left(-\eta \sum_{t=1}^{s-1} \Delta L^{h}\left(\tilde{f}^{\prime}, f^{h+1}, \zeta_{t}\right)\right)} \exp \left(-\eta \Delta L^{h}\left(\tilde{f}^{h}, f^{h+1}, \zeta_{s}\right)\right)\\
&=\mathbb{E}_{S_{s}} \mathbb{E}_{f \sim p(\cdot)} \ln \mathbb{E}_{\tilde{f}^{h} \sim \hat{q}_{s}^{h}\left(\cdot \mid f^{h+1}, S_{s-1}\right)} \exp \left(-\eta \Delta L^{h}\left(\tilde{f}^{h}, f^{h+1}, \zeta_{s}\right)\right) \\
&\leq \mathbb{E}_{S_{s}} \mathbb{E}_{f \sim p(\cdot)}\left(\mathbb{E}_{\tilde{f}^{h} \sim \hat{q}_{s}^{h}\left(\cdot \mid f^{h+1}, S_{s-1}\right)} \exp \left(-\eta \Delta L^{h}\left(\tilde{f}^{h}, f^{h+1}, \zeta_{s}\right)\right)-1\right) \leq 0
\end{align*}
where we use $\ln z \leq z - 1$. By $W_0^h = 0$, we know that 
$$W_t^h = W_0^h + \sum_{s=1}^t [W_s^h - W_{s-1}^h] \leq 0,$$
equivalently,
$$\Eb_{S_t} \Eb_{f \sim p(\cdot)} \ln \Eb_{\tilde{f}^h \sim p_0^h} \exp\left(-\eta \sum_{s=1}^t \Delta L^h(\tilde{f}^h, f^{h+1}, \zeta_s)\right) \leq 0.$$
This implies that for any $p(\cdot)$, we have
$$
\begin{aligned}
& \mathbb{E}_{S_{t-1}} \mathbb{E}_{f \sim p(\cdot)}\left[\sum_{h=1}^{H} \Phi_{t}^{h}(f)-\lambda \Delta f^{1}\left(x^{1}\right)+\ln p(f)\right] \\
&\qquad= \mathbb{E}_{S_{t-1}} \mathbb{E}_{f \sim p(\cdot)}\left[-\lambda \Delta f^{1}\left(x^{1}\right)+\alpha \eta \sum_{h=1}^{H} \sum_{s=1}^{t-1} \Delta L^{h}\left(f^{h}, f^{h+1}, \zeta_{s}\right)\right.\\
&\qquad\qquad\left.+\alpha \sum_{h=1}^{H} \ln \mathbb{E}_{\tilde{f}^{h} \sim p_{0}^{h}} \exp \left(-\eta \sum_{s=1}^{t-1} \Delta L^{h}\left(\tilde{f}^{h}, f^{h+1}, \zeta_{s}\right)\right)+\ln \frac{p(f)}{p_{0}(f)}\right] \\
&\qquad \leq \mathbb{E}_{S_{t-1}} \mathbb{E}_{f \sim p(\cdot)}\left[-\lambda \Delta f^{1}\left(x^{1}\right)+\sum_{h=1}^{H} \alpha \eta \sum_{s=1}^{t-1}\left(\mathcal{E}_{h}\left(f ; x_{s}^{h}, a_{s}^{h}, b_s^h\right)\right)^{2}+\ln \frac{p(f)}{p_{0}(f)}\right],
\end{aligned}
$$
where we use the definition of the potential function in first equality.
Since $p(\cdot)$ is arbitrary, we can take $f^h \in \cF_h(\epsilon, Q^*_{h+1})$ for all $h \in [H]$. We then have 
$$|f^h(x,a,b) - Q^*_h(x,a,b)| = |f^h(x,a,b) - \cT Q^*_{h+1}(x,a,b)| \leq \epsilon,$$
for all $(x,a,b,h) \in \cX \times \cA \times \cB \times [H]$. Then, we have
$$|\cE_h(f;x,a,b)| \leq |f^h(x,a,b) - Q^*_h(x,a,b)| + \sup_{x'} |V_{f,h+1}(x') - V^*_{h+1}(x')| \leq 2\epsilon,$$
where we use 
$$
\begin{aligned}
|V_{f,h+1}(x') - V^*_{h+1}(x')| &= |\sup_\mu \inf_\nu \Db_{\mu, \nu} f^{h+1}(x') - \sup_\mu \inf_\nu \Db_{\mu, \nu} Q^*_{h+1}(x')|\\
&\leq \sup_\mu \sup_\nu |\Db_{\mu, \nu} (f^{h+1}(x')-Q^*_{h+1}(x'))| \leq \epsilon,
\end{aligned}
$$
where the first inequality is because of 
$$|\inf_A f - \inf_A g| \leq \sup_A |f-g|.$$
By taking $p(f) = p_0(f)I(f \in \cF(\epsilon))/p_0(\cF(\epsilon))$, with $\cF(\epsilon) = \prod_h \cF_h(\epsilon, Q^*_{h+1})$, we obtain the desired result.
\end{proof}

\begin{lemma}\label{lem:mutual_information}
It holds that
\begin{equation}
\begin{aligned}
\mathbb{E}_{f \sim \hat{p}_{t}(f)} \ln \hat{p}_{t}(f) \geq & \alpha \mathbb{E}_{f \sim \hat{p}_{t}} \ln \hat{p}_{t}(f)+(1-\alpha) \mathbb{E}_{f \sim \hat{p}_{t}} \sum_{h=1}^{H} \ln \hat{p}_{t}\left(f^{h}\right) \\
\geq & \frac{\alpha}{2} \sum_{h=1}^{H} \mathbb{E}_{f \sim \hat{p}_{t}} \ln \hat{p}_{t}\left(f^{h}, f^{h+1}\right) \\
&\qquad+(1-0.5 \alpha) \mathbb{E}_{f \sim \hat{p}_{t}} \ln \hat{p}_{t}\left(f^{1}\right)+(1-\alpha) \sum_{h=2}^{H} \mathbb{E}_{f \sim \hat{p}_{t}} \ln \hat{p}_{t}\left(f^{h}\right) .
\end{aligned}
\end{equation}
\end{lemma}
\begin{proof}
To show the first inequality, we just subtract all terms of RHS from LHS to see that it is a KL-divergence which is non-negative. The second inequality is equivalent to 
$$
\mathbb{E}_{f \sim \hat{p}_{t}} \ln \hat{p}_{t}(f) \geq 0.5 \mathbb{E}_{f \sim \hat{p}_{t}} \ln \hat{p}_{t}\left(f^{1}\right)+0.5 \sum_{h=1}^{H} \mathbb{E}_{f \sim \hat{p}_{t}} \ln \hat{p}_{t}\left(f^{h}, f^{h+1}\right).
$$
This follows from 
$$
0.5 \mathbb{E}_{f \sim \hat{p}_{t}} \ln \hat{p}_{t}(f) \geq 0.5 \sum_{h=1}^{H} \mathbb{E}_{f \sim \hat{p}_{t}} \ln \hat{p}_{t}\left(f^{h}, f^{h+1}\right) I(h \text { is a odd number })
$$
and
$$
0.5 \mathbb{E}_{f \sim \hat{p}_{t}} \ln \hat{p}_{t}(f) \geq 0.5 \mathbb{E}_{f \sim \hat{p}_{t}} \ln \hat{p}_{t}\left(f^{1}\right)+0.5 \sum_{h=1}^{H} \mathbb{E}_{f \sim \hat{p}_{t}} \ln \hat{p}_{t}\left(f^{h}, f^{h+1}\right) I(h \text { is an even number })
$$
which is a result of the non-negativity of mutual information.
\end{proof}

\begin{lemma}
\label{lem:entropy}
It holds that
\begin{equation}
\begin{aligned}
&\mathbb{E}_{S_{t-1}} \mathbb{E}_{f \sim \hat{p}_{t}}\left(\sum_{h=1}^{H} \Phi_{t}^{h}(f)-\lambda \Delta f^{1}\left(x^{1}\right)+\ln \hat{p}_{t}(f)\right) \\
&\qquad\geq \underbrace{\mathbb{E}_{S_{t-1}} \mathbb{E}_{f \sim \hat{p}_{t}}\left[-\lambda \Delta f^{1}\left(x^{1}\right)+(1-0.5 \alpha) \ln \frac{\hat{p}_{t}\left(f^{1}\right)}{p_{0}^{1}\left(f^{1}\right)}\right]}_{A} \\
&\qquad\qquad+\sum_{h=1}^{H} \underbrace{0.5 \alpha \mathbb{E}_{S_{t-1}} \mathbb{E}_{f \sim \hat{p}_{t}}\left[\eta \sum_{s=1}^{t-1} 2 \Delta L^{h}\left(f^{h}, f^{h+1}, \zeta_{s}\right)+\ln \frac{\hat{p}_{t}\left(f^{h}, f^{h+1}\right)}{p_{0}^{h}\left(f^{h}\right) p_{0}^{h+1}\left(f^{h+1}\right)}\right]}_{B_{h}} \\
&\qquad\qquad+\sum_{h=1}^{H} \underbrace{\mathbb{E}_{S_{t-1}} \mathbb{E}_{f \sim \hat{p}_{t}}\left[\alpha \ln \mathbb{E}_{\tilde{f}^{h} \sim p_{0}^{h}} \exp \left(-\eta \sum_{s=1}^{t-1} \Delta L^{h}\left(\tilde{f}^{h}, f^{h+1}, \zeta_{s}\right)\right)+(1-\alpha) \ln \frac{\hat{p}_{t}\left(f^{h+1}\right)}{p_{0}^{h+1}\left(f^{h+1}\right)}\right]}_{C_{h}} .
\end{aligned}
\end{equation}
\end{lemma}
\begin{proof}
We use the definition of the potential function and apply Lemma~\ref{lem:mutual_information}. The desired result follows from some calculations.
\end{proof}

\begin{lemma}
\label{lem:lower_bound}
If $\eta \beta^2\leq 0.4$, it holds that
\begin{equation}
    \begin{aligned}
    A \geq -\lambda \Eb_{S_{t-1}} \Eb_{f_t \sim \hat{p}_t} \Delta f_t^1(x^1),
    \end{aligned}
\end{equation}
\begin{equation}
    \begin{aligned}
B_{h} \geq 0.25 \alpha \eta \sum_{s=1}^{t-1} \mathbb{E}_{S_{t-1}} \mathbb{E}_{f \sim \hat{p}_{t}} \mathbb{E}_{\pi_{s}}\left(\mathcal{E}_{h}\left(f ; x_{s}^{h}, a_{s}^{h}, b_s^h\right)\right)^{2}
\end{aligned}
\end{equation}
\begin{equation}
    \begin{aligned}
C_{h} \geq-\alpha \eta \epsilon(2 b+\epsilon)(t-1)-\kappa^{h}_1(\alpha, \epsilon).
\end{aligned}
\end{equation}
\end{lemma}
\begin{proof}
The bound of $A$ comes from the non-negativity of KL-divergence and $\alpha \in (0,1]$. To prove the lower bound of $B_h$, we define 
$$
\begin{aligned} \xi_{s}^{h}\left(f^{h}, f^{h+1}, \zeta_{s}\right)=&-2 \eta \Delta L^{h}\left(f^{h}, f^{h+1}, \zeta_{s}\right) -\ln \mathbb{E}_{x_s^{h+1} \sim \Pb^h\left(\cdot \mid x_{s}^{h}, a_{s}^{h}, b_s^h\right)} \exp \left(-2 \eta \Delta L^{h}\left(f^{h}, f^{h+1}, \zeta_{s}\right)\right).
\end{aligned}
$$
Then, for all $h \in [H]$, we have 
$$
\mathbb{E}_{S_{t-1}} \exp \left(\sum_{s=1}^{t-1} \xi_{s}^{h}\left(f^{h}, f^{h+1}, \zeta_{s}\right)\right)=1,
$$
according to Lemma~\ref{lem:martingale}. Then, by Lemma~\ref{lem: Gibbs_var}, we have
$$
\begin{aligned}
& \mathbb{E}_{f \sim \hat{p}_{t}}\left[\sum_{s=1}^{t-1}-\xi_{s}^{h}\left(f^{h}, f^{h+1}, \zeta_{s}\right)+\ln \frac{\hat{p}_{t}\left(f^{h}, f^{h+1}\right)}{p_{0}^{h}\left(f^{h}\right) p_{0}^{h+1}\left(f^{h+1}\right)}\right] \\
&\qquad \geq \inf _{p} \mathbb{E}_{f \sim p}\left[\sum_{s=1}^{t-1}-\xi_{s}^{h}\left(f^{h}, f^{h+1}, \zeta_{s}\right)+\ln \frac{p\left(f^{h}, f^{h+1}\right)}{p_{0}^{h}\left(f^{h}\right) p_{0}^{h+1}\left(f^{h+1}\right)}\right] \\
&\qquad=-\ln \mathbb{E}_{f^{h} \sim p_{0}^{h}} \mathbb{E}_{f^{h+1} \sim p_{0}^{h+1}} \exp \left(\sum_{s=1}^{t-1} \xi_{s}^{h}\left(f^{h}, f^{h+1}, \zeta_{s}\right)\right),
\end{aligned}
$$
where we use the fact that Lemma~\ref{lem: Gibbs_var} implies that the $\inf$ is achieved at 
$$
p\left(f^{h}, f^{h+1}\right) \propto p_{0}^{h}\left(f^{h}\right) p_{0}^{h+1}\left(f^{h+1}\right) \exp \left(\sum_{s=1}^{t-1} \xi_{s}^{h}\left(f^{h}, f^{h+1}, \zeta_{s}\right)\right),
$$
and the expectation is equal to
$$-\Eb_{p\left(f^{h}, f^{h+1}\right)}\sum_{s=1}^{t-1}\xi_s^h(f^h,f^{h+1},\zeta_s) + \Eb_{p\left(f^{h}, f^{h+1}\right)} \ln \frac{\exp(\sum_{s=1}^{t-1}\xi_s^h(f^h,f^{h+1},\zeta_s))}{c} = -\ln c$$
where $c = \mathbb{E}_{f^{h} \sim p_{0}^{h}} \mathbb{E}_{f^{h+1} \sim p_{0}^{h+1}} \exp \left(\sum_{s=1}^{t-1} \xi_{s}^{h}\left(f^{h}, f^{h+1}, \zeta_{s}\right)\right)$ is the normalized constant.
It then follows that
$$
\begin{aligned}
& \mathbb{E}_{S_{t-1}} \mathbb{E}_{f \sim \hat{p}_{t}}\left[\sum_{s=1}^{t-1}-\xi_{s}^{h}\left(f^{h}, f^{h+1}, \zeta_{s}\right)+\ln \frac{\hat{p}_{t}\left(f^{h}, f^{h+1}\right)}{p_{0}^{h}\left(f^{h}\right) p_{0}^{h+1}\left(f^{h+1}\right)}\right] \\
%= & \mathbb{E}_{S_{t-1}} \mathbb{E}_{f \sim \hat{p}_{t}}\left[\sum_{s=1}^{t-1} 2 \eta \Delta L^{h}\left(f^{h}, f^{h+1}, \zeta_{s}\right) + \ln \mathbb{E}_{x_s^{h+1} \sim \Pb^h\left(\cdot \mid x_{s}^{h}, a_{s}^{h}, b_s^h\right)} \exp \left(-2 \eta \Delta L^{h}\left(f^{h}, f^{h+1}, \zeta_{s}\right)\right) +\ln \frac{\hat{p}_{t}\left(f^{h}, f^{h+1}\right)}{p_{0}^{h}\left(f^{h}\right) p_{0}^{h+1}\left(f^{h+1}\right)}\right] \\
&\qquad\geq -\mathbb{E}_{S_{t-1}} \ln \mathbb{E}_{f^{h} \sim p_{0}^{h}} \mathbb{E}_{f^{h+1} \sim p_{0}^{h+1}} \exp \left(\sum_{s=1}^{t-1} \xi_{s}^{h}\left(f^{h}, f^{h+1}, \zeta_{s}\right)\right) \\
&\qquad\geq -\ln \mathbb{E}_{f^{h} \sim p_{0}^{h}} \mathbb{E}_{f^{h+1} \sim p_{0}^{h+1}} \mathbb{E}_{S_{t-1}} \exp \left(\sum_{s=1}^{t-1} \xi_{s}^{h}\left(f^{h}, f^{h+1}, \zeta_{s}\right)\right)=0,
\end{aligned}
$$
where we use the above result in the first inequality and use the convexity of $-\ln(\cdot)$ in the last inequality. We then have 
$$
\begin{aligned}
B_h &=0.5 \alpha \mathbb{E}_{S_{t-1}} \mathbb{E}_{f \sim \hat{p}_{t}}\left[\eta \sum_{s=1}^{t-1} 2 \Delta L^{h}\left(f^{h}, f^{h+1}, \zeta_{s}\right)+\ln \frac{\hat{p}_{t}\left(f^{h}, f^{h+1}\right)}{p_{0}^{h}\left(f^{h}\right) p_{0}^{h+1}\left(f^{h+1}\right)}\right]\\
&\geq 0.5\alpha \Eb_{S_{t-1}} \Eb_{f \sim \hat{p}_t} \sum_{s=1}^{t-1}-\ln \mathbb{E}_{x_s^{h+1} \sim \Pb^h\left(\cdot \mid x_{s}^{h}, a_{s}^{h}, b_s^h\right)} \exp \left(-2 \eta \Delta L^{h}\left(f^{h}, f^{h+1}, \zeta_{s}\right)\right)\\
&\geq -0.5 \alpha \eta \sum_{s=1}^{t-1} \frac{1}{2}(\cE_h(f;x_s^h, a_s^h, r_s^h))^2,
\end{aligned}
$$
where we use the definition of $\xi_{s}^{h}\left(f^{h}, f^{h+1}, \zeta_{s}\right)$ in the first inequality and  we use Lemma~\ref{lem:ln_expectation} in the last step. 

We now turn to the lower bound of $C_h$. We have 
$$
\begin{aligned}
& \mathbb{E}_{f \sim \hat{p}_{t}}\left[\alpha \ln \mathbb{E}_{\tilde{f}^{h} \sim p_{0}^{h}} \exp \left(-\eta \sum_{s=1}^{t-1} \Delta L^{h}\left(\tilde{f}^{h}, f^{h+1}, \zeta_{s}\right)\right)+(1-\alpha) \ln \frac{\hat{p}_{t}\left(f^{h+1}\right)}{p_{0}^{h+1}\left(f^{h+1}\right)}\right] \\
&\qquad \geq (1-\alpha) \inf _{p^h} \mathbb{E}_{f \sim p^h}\left[\frac{\alpha}{1-\alpha} \ln \mathbb{E}_{\tilde{f}^{h} \sim p_{0}^{h}} \exp \left(-\eta \sum_{s=1}^{t-1} \Delta L^{h}\left(\tilde{f}^{h}, f^{h+1}, \zeta_{s}\right)\right)+\ln \frac{p^h\left(f^{h+1}\right)}{p_{0}^{h+1}\left(f^{h+1}\right)}\right] \\
&\qquad=-(1-\alpha) \ln \mathbb{E}_{f^{h+1} \sim p_{0}^{h+1}}\left(\mathbb{E}_{f^{h} \sim p_{0}^{h}} \exp \left(-\eta \sum_{s=1}^{t-1} \Delta L^{h}\left(f^{h}, f^{h+1}, \zeta_{s}\right)\right)\right)^{-\alpha /(1-\alpha)},
\end{aligned}
$$
where we use the fact that the $\inf$ is achieved at 
$$
p^h\left(f^{h+1}\right) \propto p_{0}^{h+1}\left(f^{h+1}\right)\left(\mathbb{E}_{f^{h} \sim p_{0}^{h}} \exp \left(-\eta \sum_{s=1}^{t-1} \Delta L^{h}\left(f^{h}, f^{h+1}, \zeta_{s}\right)\right)\right)^{-\alpha /(1-\alpha)}.
$$
We now consider a fixed $f^h \in \cF_h(\epsilon, f^{h+1})$. It holds that 
$$
\left|\Delta L^{h}\left(f^{h}, f^{h+1}, \zeta_{s}\right)\right| \leq\left(\mathcal{E}_{h}\left(f, x_{s}^{h}, a_{s}^{h}\right)\right)^{2}+2 \beta\left|\mathcal{E}_{h}\left(f, x_{s}^{h}, a_{s}^{h}\right)\right| \leq \epsilon(2 \beta+\epsilon)
$$
To show this, we recall the definition 
$$
\begin{aligned}
\Delta L^{h}\left(f^{h}, f^{h+1} ; \zeta_{s}\right)=&\left(f^{h}\left(x_{s}^{h}, a_{s}^{h}, b_s^h\right)-r_{s}^{h}-V_{f,h+1}\left(x_{s}^{h+1}\right)\right)^{2} \\
&\qquad-\left(\mathcal{T}_{h} f^{h+1}\left(x_{s}^{h}, a_{s}^{h}, b_s^h\right)-r_{s}^{h}-V_{f,{h+1}}\left(x_{s}^{h+1}\right)\right)^{2},
\end{aligned}
$$
and we subtract and add $\cT_h f^{h+1}(x_s^h,a_s^h)$ inside the first term to obtain
$$\Delta L^h(f^h, f^{h+1}, \zeta_s) = \cE_h(f,x_s^h,a_s^h)^2 + 2\cE_h(f,x_s^h,a_s^h)(\cT^*_hf^{h+1}(x_s^h, a_s^h)-r^h_s-f^{h+1}(x_s^{h+1})).$$
It follows that 
$$
\mathbb{E}_{f^{h} \sim p_{0}^{h}} \exp \left(-\eta \sum_{s=1}^{t-1} \Delta L^{h}\left(f^{h}, f^{h+1}, \zeta_{s}\right)\right) \leq p_{0}^{h}\left(\mathcal{F}_{h}\left(\epsilon, f^{h+1}\right)\right) \exp (-\eta(t-1)(2 \beta+\epsilon) \epsilon).
$$
Thus, we have
\begin{align*}
C_h &\geq \alpha \Eb_{S_{t-1}} \ln \Eb_{f^{h+1} \sim p_0^{h+1}} p_0^h(\cF_h(\epsilon, f^{h+1})) \exp(-\eta(t-1)(2\beta+\epsilon)\epsilon)\\
&= -\alpha\eta\epsilon(2\beta+\epsilon)(t-1) + \alpha \Eb_{S_{t-1}} \ln \Eb_{f^{h+1} \sim p_0^{h+1}} p_0^h(\cF_h(\epsilon, f^{h+1})) \\
&\geq -\alpha\eta\epsilon(2\beta+\epsilon)(t-1) - \kappa^h_1(\alpha, \epsilon)
\end{align*}
where we use the definition
$$\kappa^{h}_1(\alpha, \epsilon)=(1-\alpha) \ln \mathbb{E}_{f^{h+1} \sim p_{0}^{h+1}} p_{0}^{h}\left(\mathcal{F}_{h}\left(\epsilon, f^{h+1}\right)\right)^{-\alpha /(1-\alpha)}.$$
\end{proof}
We are ready to prove Theorem~\ref{thm:adversarial}.
\begin{proof}[Proof of Theorem~\ref{thm:adversarial}.]
Let $\pi_t$ denote the distribution induced by $\mu_t \times \nu_t$ and define
$$\delta_t^h = \lambda \cE_h(f_t; x_t^h, a_t^h, b_t^h) - 0.25 \alpha \eta \sum_{s=1}^{t-1} \Eb_{\pi_s} \left(\cE_h(f_t; x_t^h, a_t^h, b_t^h)\right)^2.$$
Then, we have
$$\sum_{t=1}^T \Eb_{S_{t-1}} \Eb_{f_t \sim \hat{p}_t} \Eb_{g_t \sim \hat{p}^{\mu_t}_t} \Eb_{\zeta_t \sim \pi_t} \sum_{h=1}^H \delta_t^h \leq \frac{\lambda^2}{\alpha \eta}dc(\cF, MG, T, 0.25\alpha\eta/\lambda).$$
For arbitrary $\nu_t$ induced by $\mu_{f_t}$ and $g_t$, according to the value-decomposition Lemma~\ref{lem:value_decom_main} we have
$$
\begin{aligned}
&\Eb_{S_{t-1}} \Eb_{f_t \sim \hat{p}_t} \Eb_{g_t \sim \hat{p}^{\mu_t}_t} \lambda (V_1^*(x^1) - V_1^{\mu_t, \nu_t}(x^1)) - \Eb_{S_{t-1}} \Eb_{f_t \sim \hat{p}_t}\Eb_{g_t \sim \hat{p}^{\mu_t}_t} \Eb_{\zeta_t \sim \pi_t} \sum_{h=1}^H \delta_t^h\\
&\leq -\lambda \Eb_{S_{t-1}} \Eb_{f_t \sim \hat{p}_t} \Delta f_t^1(x^1) + 0.25\alpha \eta \sum_{h=1}^H \sum_{s=1}^{t-1} \Eb_{S_{t-1}} \Eb_{f_t \sim \hat{p}_t}\Eb_{\pi_s} \left(\cE_h(f_t; x_t^h, a_t^h, b_t^h)\right)^2\\
&\leq \mathbb{E}_{S_{t-1}} \mathbb{E}_{f \sim \hat{p}_{t}}\left(\sum_{h=1}^{H} \Phi_{t}^{h}(f)-\lambda \Delta f^{1}\left(x^{1}\right)+\ln \hat{p}_{t}(f)\right)+\alpha \eta \epsilon(2 \beta+\epsilon)(t-1) H+\sum_{h=1}^{H} \kappa^{h}_1(\alpha, \epsilon)\\
&= \mathbb{E}_{S_{t-1}} \inf _{p} \mathbb{E}_{f \sim p}\left(\sum_{h=1}^{H} \Phi_{t}^{h}(f)-\lambda \Delta f^{1}\left(x^{1}\right)+\ln p(f)\right)+\alpha \eta \epsilon(2 \beta+\epsilon)(t-1) H+\sum_{h=1}^{H} \kappa^{h}_1(\alpha, \epsilon) \\
&\leq\lambda \epsilon+\alpha \eta \epsilon(\epsilon+4 \epsilon+2 \beta)(t-1) H-\sum_{h=1}^{H} \ln p_{0}^{h}(\mathcal{F}( \epsilon, Q_{h+1}^{*}))+\sum_{h=1}^{H} \kappa^{h}_1(\alpha, \epsilon),
\end{aligned}
$$
where the first inequality also uses the definition of $\Delta f_t^1(x^1)$; the second inequality comes from Lemma~\ref{lem:entropy} and Lemma~\ref{lem:lower_bound}; the equality is because Lemma~\ref{lem:key}, and the last step comes from Lemma~\ref{lem:upper}.
Summing over $t$, we obtain that 
$$
\begin{aligned}
&\sum_{t=1}^T \Eb_{S_{t-1}} \Eb_{f_t \sim \hat{p}_t}\Eb_{g_t \sim \hat{p}^{\mu_t}_t} \left(V_1^*(x^1) - V_1^{\mu_t, \nu_t}(x^1)\right) \\
&\leq \epsilon T + \frac{1}{\lambda} \alpha \eta (5\epsilon + 2\beta) \frac{T(T-1)}{2} H - \frac{T}{\lambda} \sum_{h=1}^{H} \ln p_{0}^{h}(\mathcal{F}( \epsilon, Q_{h+1}^{*})) + \frac{T}{\lambda} \sum_{h=1}^{H} \kappa^{h}(\alpha, \epsilon) + \frac{\lambda}{\alpha \eta} dc(\cF, MG, T, 0.25\alpha \eta /\lambda)\\
&\leq O(\beta\sqrt{dc(\cF,MG,T)\kappa(\beta/T^2)T} + dc(\cF,MG,T)). 
\end{aligned}
$$
Here in the last step, we first let $\alpha \to 1^-$ and note that
\begin{align*}
-\ln p_{0}^{h}\left(\mathcal{F}{(} \epsilon, Q_{h+1}^{*}\right) &\leq \kappa^{h}_1(1, \epsilon), \\
-\sum_{h=1}^{H} \ln p_{0}^{h}\left(\mathcal{F}( \epsilon, Q_{h+1}^{*}\right)+\sum_{h=1}^{H} \kappa^{h}_1(1, \epsilon) &\leq 2 \kappa(\epsilon).
\end{align*}
Then, we take $\epsilon = \frac{\beta}{T^2}$, $\lambda = \sqrt{\frac{T\kappa(\beta/T^2)}{\beta^2dc(\cF,MG,T)}}$, and $\eta  = \frac{1}{4\beta^2}$.  This concludes the proof.
\end{proof}

\section{Proof of the Theorem~\ref{thm:booster agent}}
In this section, we provide a proof for Theorem~\ref{thm:booster agent}. 

\begin{lemma} \label{lem:moment2} For any max-player's policy $\mu$ and all functions $g\in \cF$, we have
$$
\mathbb{E}_{x_s^{h+1} \sim \Pb^{h}\left(\cdot \mid x_{s}^{h}, a_{s}^{h},b_s^h\right)} \Delta L^{h}_\mu\left(g^{h}, g^{h+1}, \zeta_{s}\right)=\left(\mathcal{E}^\mu_{h}\left(g ; x_{s}^{h}, a_{s}^{h}, b_s^h\right)\right)^{2}
$$
and
$$
\mathbb{E}_{x_s^{h+1} \sim \Pb^{h}\left(\cdot \mid x_{s}^{h}, a_{s}^{h}, b_s^h\right)} \Delta L^{h}_\mu\left(g^{h}, g^{h+1}, \zeta_{s}\right)^{2} \leq \frac{4 \beta^{2}}{3}\left(\mathcal{E}^\mu_{h}\left(g ; x_{s}^{h}, a_{s}^{h}, b_s^h\right)\right)^{2}
$$
\end{lemma}
\begin{proof}
The proof of this lemma only employs the Markov property of the transition and the range of function $g \in \cF$. By replacing the notations in the proof of Lemma~\ref{lem:moment}, we conclude the proof.
\end{proof}

\begin{lemma} \label{lem:ln_expectation2}
Letting $\eta \beta^2 \leq 0.8$, then for all functions $g \in \cF$ and any max-player's policy $\mu$, we have
$$
\begin{aligned}
& \ln \mathbb{E}_{x_s^{h+1} \sim \Pb^{h}\left(\cdot \mid x_{s}^{h}, a_{s}^{h}, b_s^h\right)} \exp \left(-\eta \Delta L^{h}_\mu\left(g^{h}, g^{h+1}, \zeta_{s}\right)\right) \\
&\qquad\leq \mathbb{E}_{x_s^{h+1} \sim \Pb^{h}\left(\cdot \mid x_{s}^{h}, a_{s}^{h}, b_s^h\right)} \exp \left(-\eta \Delta L^{h}_\mu\left(g^{h}, g^{h+1}, \zeta_{s}\right)\right)-1 \\
&\qquad\leq -0.25 \eta \left(\mathcal{E}^\mu_{h}\left(f ; x_{s}^{h}, a_{s}^{h}, b_s^h\right)\right)^{2} .
\end{aligned}
$$
\end{lemma}
\begin{proof}
The proof of this lemma only employs the range of function $g\in \cF$. By replacing the notations in the proof of Lemma~\ref{lem:ln_expectation}, we conclude the proof.
\end{proof}

\begin{lemma}
It holds that
$$
\begin{aligned}
&\mathbb{E}_{S_{t-1}} \Eb_{f_t \sim \hat{p}_t} \inf_{p} \mathbb{E}_{g \sim p(\cdot)}\left[\sum_{h=1}^{H} {\Phi}_{t}^{h}(g, \mu_t)-\lambda \Delta g^{1}_{\mu_t} \left(x^{1}\right)+\ln p(g)\right] \\
&\qquad\leq \lambda \epsilon+4 \alpha \eta(t-1) H \epsilon^{2}-\Eb_{S_{t-1}} \Eb_{f_t \sim \hat{p}_t} \sum_{h=1}^{H} \ln p_{0}^{h}\left(\mathcal{F}^{\mu_t}_{h}\left(\epsilon, Q_{h+1}^{\mu_t,\dagger}\right)\right)
\end{aligned}
$$
\end{lemma}
\begin{proof}
Consider any fixed $g \in \cF$. For any $\tilde{g}^h \in \cF^h$ that depends on $S_{s-1}$ and $\mu_{f_s}$, and for any $\mu_f$ we obtain from Lemma~\ref{lem:ln_expectation} that 
$$
\mathbb{E}_{\zeta_{s}} \exp \left(-\eta \Delta L^{h}_{\mu_f}\left(\tilde{g}^{h}, g^{h+1}, \zeta_{s}\right)\right)-1 \leq-0.25 \eta \mathbb{E}_{\zeta_{s}}\left(\tilde{g}^{h}(x, a)-\mathcal{T}_{h}^{\mu_f} g^{h+1}(x, a, b)\right)^{2} \leq 0.
$$
We now fix some $t$. For all $s \leq t$, we define
$$W_s^h := \Eb_{S_s} \Eb_{f \sim \hat{p}_{t+1}} \Eb_{g \sim p(\cdot)} \ln \Eb_{\tilde{g}^h \sim p_0^h} \exp\left(-\eta \sum_{\ell=1}^s \Delta L^h_{\mu_f}(\tilde{g}^h, g^{h+1}, \zeta_\ell)\right),$$
and recall that 
$$
\hat{q}_{t}^{h}\left(\tilde{g}^{h} \mid g^{h+1}, \mu_f, S_{t-1}\right)=\frac{p_0^h(\tilde{g}^h)\exp \left(-\eta \sum_{s=1}^{t-1} \Delta L^{h}_{\mu_f}\left(\tilde{g}^{h}, g^{h+1}, \zeta_{s}\right)\right)}{\mathbb{E}_{\tilde{g}' \sim p_{0}^{h}} \exp \left(-\eta \sum_{s=1}^{t-1} \Delta L^{h}_{\mu_f}\left(\tilde{g}^{\prime}, g^{h+1}, \zeta_{s}\right)\right)}.
$$
We have 
$$
\begin{gathered}
W_{s}^{h}-W_{s-1}^{h}=\mathbb{E}_{S_{s}} \mathbb{E}_{f \sim \hat{p}_{t+1}(\cdot)} \ln \mathbb{E}_{\tilde{g}^{h} \sim \hat{q}_{s}^{h}\left(\cdot \mid g^{h+1}, \mu_f, S_{s-1}\right)} \exp \left(-\eta \Delta L^{h}_{\mu_f}\left(\tilde{g}^{h}, g^{h+1}, \zeta_{s}\right)\right) \\
\leq \mathbb{E}_{S_{s}} \mathbb{E}_{f \sim \hat{p}_t(\cdot)}\left(\mathbb{E}_{\tilde{g}^{h} \sim \hat{q}_{s}^{h}\left(\cdot \mid g^{h+1},\mu_f, S_{s-1}\right)} \exp \left(-\eta \Delta L^{h}_{\mu_f}\left(\tilde{g}^{h}, g^{h+1}, \zeta_{s}\right)\right)-1\right) \leq 0
\end{gathered}
$$
where we use $\ln z \leq z - 1$. By $W_0^h = 0$, we know that 
$$W_t^h = W_0^h + \sum_{s=1}^t [W_s^h - W_{s-1}^h] \leq 0,$$
equivalently,
$$\Eb_{S_t} \Eb_{f \sim \hat{p}_{t+1}(\cdot)} \Eb_{g \sim p(\cdot)}\ln \Eb_{\tilde{g}^h \sim p_0^h} \exp\left(-\eta \sum_{s=1}^t \Delta L^h_{\mu_f}(\tilde{g}^h, g^{h+1}, \zeta_s)\right) \leq 0.$$
Note that $t$ is arbitrary. This implies that for any $p(\cdot)$ and any $t$, we have
$$
\begin{aligned}
& \mathbb{E}_{S_{t-1}} \mathbb{E}_{f \sim \hat{p}_t(\cdot)} \Eb_{g \sim p(\cdot)}\left[\sum_{h=1}^{H} \Phi_{t}^{h}(g, \mu_f)-\lambda \Delta g^{1}_{\mu_f}\left(x^{1}\right)+\ln p(g)\right] \\
&\qquad= \mathbb{E}_{S_{t-1}} \mathbb{E}_{f \sim \hat{p}_t(\cdot)} \Eb_{g \sim p(\cdot)}\left[-\lambda \Delta g^{1}_{\mu_f}\left(x^{1}\right)+\alpha \eta \sum_{h=1}^{H} \sum_{s=1}^{t-1} \Delta L^{h}_{\mu_f}\left(g^{h}, g^{h+1}, \zeta_{s}\right)\right.\\
&\qquad \qquad \left.+\alpha \sum_{h=1}^{H} \ln \mathbb{E}_{\tilde{g}^{h} \sim p_{0}^{h}} \exp \left(-\eta \sum_{s=1}^{t-1} \Delta L^{h}_{\mu_f}\left(\tilde{g}^{h}, g^{h+1}, \zeta_{s}\right)\right)+\ln \frac{p(g)}{p_{0}(g)}\right] \\
&\qquad\leq \mathbb{E}_{S_{t-1}} \mathbb{E}_{f \sim \hat{p}_t(\cdot)} \Eb_{g \sim p(\cdot)}\left[-\lambda \Delta g^{1}_{\mu_f}\left(x^{1}\right)+\sum_{h=1}^{H} \alpha \eta \sum_{s=1}^{t-1}\left(\mathcal{E}^{\mu_f}_{h}\left(g ; x_{s}^{h}, a_{s}^{h},b_s^h\right)\right)^{2}+\ln \frac{p(g)}{p_{0}(g)}\right].
\end{aligned}
$$
Since $p(\cdot)$ is arbitrary, we can take $g^h \in \cF_h^\mu(\epsilon, Q^{\mu, \dagger}_{h+1})$ for all $h \in [H]$. We need to show that $g^h$ admits a small $\cT_h^\mu$-Bellman-residual. We have
$$|g^h(x,a,b) - Q^{\mu, \dagger}_{h}(x,a,b)| = |g^{h}(x,a,b) - \cT^\mu_h Q^{\mu, \dagger}_{h+1}(x,a,b)| \leq \epsilon, $$
for all $(x,a,b,h) \in \cX \times \cA \times \cB \times [H]$. Then, we have
$$|\cE^\mu_h(g;x,a,b)| \leq |g^h(x,a,b) - Q^{\mu, \dagger}_{h}(x,a,b)| + \sup_{x'} |V^\mu_{g,h+1}(x') - V^{\mu,\dagger}_{h+1}(x')| \leq 2\epsilon,$$
where we use 
$$
\begin{aligned}
|V^\mu_{g,h+1}(x') - V^{\mu,\dagger}_{h+1}(x')| &= |\inf_\nu \Db_{\mu, \nu} g(x') - \inf_\nu \Db_{\mu, \nu} Q^{\mu,\dagger}_{h+1}(x')|\\
&\leq \sup_\nu |\Db_{\mu, \nu} (g(x')-Q^{\mu,\dagger}_{h+1}(x'))|\leq \epsilon,
\end{aligned}
$$
{where we use the fact that 
$$|\inf_A f - \inf_A g| \leq \sup_A |f-g|.$$}
By taking $p(f) = p_0(f)I(f \in \cF(\epsilon, \mu_f))/p_0(\cF(\epsilon, \mu_f))$, with $\cF(\epsilon, \mu_f) = \prod_h \cF^{\mu_f}_h(\epsilon, Q^{\mu_f,\dagger}_{h+1})$, we obtain the desired result.
\end{proof}

\begin{lemma}\label{lem:mutual_information2}
For any max-player's policy $\mu$ that is induced by some $f \in \cF$, we have 
\begin{equation}
\begin{aligned}
\mathbb{E}_{g \sim \hat{p}^\mu_{t}(g)} \ln \hat{p}^\mu_{t}(g) \geq & \alpha \mathbb{E}_{g \sim \hat{p}^\mu_{t}} \ln \hat{p}^\mu_{t}(g)+(1-\alpha) \mathbb{E}_{g \sim \hat{p}^\mu_{t}} \sum_{h=1}^{H} \ln \hat{p}^\mu_{t}\left(g^{h}\right) \\
\geq & \frac{\alpha}{2} \sum_{h=1}^{H} \mathbb{E}_{g \sim \hat{p}^\mu_{t}} \ln \hat{p}^\mu_{t}\left(g^{h}, g^{h+1}\right) \\
&+(1-0.5 \alpha) \mathbb{E}_{g \sim \hat{p}^\mu_{t}} \ln \hat{p}^\mu_{t}\left(g^{1}\right)+(1-\alpha) \sum_{h=2}^{H} \mathbb{E}_{g \sim \hat{p}^\mu_{t}} \ln \hat{p}^\mu_{t}\left(g^{h}\right) .
\end{aligned}
\end{equation}
\end{lemma}
\begin{proof}
The proof of this lemma only relies on the non-negativity of mutual information and KL-divergence. By replacing the notations in the proof of Lemma~\ref{lem:mutual_information}, we conclude the proof.
\end{proof}

\begin{lemma}
\label{lem:entropy2}
It holds that 
\begin{equation}
\begin{aligned} 
&\mathbb{E}_{S_{t-1}} \mathbb{E}_{f_t \sim \hat{p}_{t}} \Eb_{g \sim \hat{p}_t^\mu} \left(\sum_{h=1}^{H} \Phi_{t}^{h}(g)-\lambda \Delta g_{\mu_t}^{1}\left(x^{1}\right)+\ln \hat{p}^{\mu_t}_{t}(g)\right) \\
&\geq \underbrace{\mathbb{E}_{S_{t-1}} \mathbb{E}_{f_t \sim \hat{p}_{t}} \Eb_{g \sim \hat{p}_t^\mu} \left[-\lambda \Delta g^{1}_{\mu_t}\left(x^{1}\right)+(1-0.5 \alpha) \ln \frac{\hat{p}^{\mu_t}_{t}\left(g^{1}\right)}{p_{0}^{1}\left(g^{1}\right)}\right]}_{A'} \\
&\qquad+\sum_{h=1}^{H} \underbrace{0.5 \alpha \mathbb{E}_{S_{t-1}} \mathbb{E}_{f_t \sim \hat{p}_{t}}\Eb_{g \sim \hat{p}_t^{\mu_t}}\left[\eta \sum_{s=1}^{t-1} 2 \Delta L^{h}_{\mu_t}\left(g^{h}, g^{h+1}, \zeta_{s}\right)+\ln \frac{\hat{p}^{\mu_t}_{t}\left(g^{h}, g^{h+1}\right)}{p_{0}^{h}\left(g^{h}\right) p_{0}^{h+1}\left(g^{h+1}\right)}\right]}_{B_{h}'} \\
&\qquad+\sum_{h=1}^{H} \underbrace{\mathbb{E}_{S_{t-1}} \mathbb{E}_{f_t \sim \hat{p}_{t}}\Eb_{g \sim \hat{p}_t^{\mu_t}}\left[\alpha \ln \mathbb{E}_{\tilde{g}^{h} \sim p_{0}^{h}} \exp \left(-\eta \sum_{s=1}^{t-1} \Delta L^{h}_{\mu_t}\left(\tilde{g}^{h}, g^{h+1}, \zeta_{s}\right)\right)+(1-\alpha) \ln \frac{\hat{p}^{\mu_t}_{t}\left(g^{h+1}\right)}{p_{0}^{h+1}\left(g^{h+1}\right)}\right]}_{C_{h}'} .
\end{aligned}
\end{equation}
\end{lemma}
\begin{proof}
We use the definition of the potential function and apply Lemma~\ref{lem:mutual_information2} (note that it is valid for any $\mu_f, f \in \cF$).
\end{proof}

\begin{lemma}
\label{lem:lower_bound2}
If $\eta \beta^2\leq 0.4$, it holds that
\begin{equation}
    \begin{aligned}
    A' \geq -\lambda \Eb_{S_{t-1}} \Eb_{f_t \sim \hat{p}_t} \Eb_{g \sim \hat{p}_t^{\mu_t}} \Delta g_{\mu_t}^1(x^1),
    \end{aligned}
\end{equation}
\begin{equation}
    \begin{aligned}
B'_{h} \geq 0.25 \alpha \eta \sum_{s=1}^{t-1} \mathbb{E}_{S_{t-1}} \mathbb{E}_{f \sim \hat{p}_{t}} \Eb_{g \sim \hat{p}^{\mu_t}_t} \mathbb{E}_{\pi_{s}}\left(\mathcal{E}_{h}^{\mu_t} \left(f ; x_{s}^{h}, a_{s}^{h}, b_s^h\right)\right)^{2}
\end{aligned}
\end{equation}
\begin{equation}
    \begin{aligned}
C'_{h} \geq-\alpha \eta \epsilon(2 b+\epsilon)(t-1)-\Eb_{S_{t-1}} \Eb_{f_t \sim \hat{p}_t} \kappa^{h}_{\mu_t}(\alpha, \epsilon).
\end{aligned}
\end{equation}
\end{lemma}
\begin{proof}
The bound of $A'$ comes from the non-negativity of KL-divergence and $\alpha \in (0,1]$. To prove the lower bound of $B_h'$, we define 
$$
\begin{aligned} \xi_{s}^{h}\left(g^{h}, g^{h+1}, \mu_t, \zeta_{s}\right)=&-2 \eta \Delta L^{h}_{\mu_t}\left(g^{h}, g^{h+1}, \zeta_{s}\right) -\ln \mathbb{E}_{x_s^{h+1} \sim \Pb^h\left(\cdot \mid x_{s}^{h}, a_{s}^{h}, b_s^h\right)} \exp \left(-2 \eta \Delta L^{h}_{\mu_t}\left(g^{h}, g^{h+1}, \zeta_{s}\right)\right),
\end{aligned}
$$
where $\mu_t$ is an arbitrary policy induced by some $f_t \in \cF$. Then, for all $h \in [H]$, we have 
$$
\mathbb{E}_{S_{t-1}} \exp \left(\sum_{s=1}^{t-1} \xi_{s}^{h}\left(g^{h}, g^{h+1}, \zeta_{s}\right)\right)=1,
$$
according to Lemma~\ref{lem:martingale}. Then, by Lemma~\ref{lem: Gibbs_var}, we have
$$
\begin{aligned}
& \mathbb{E}_{g \sim \hat{p}^{\mu_t}_{t}}\left[\sum_{s=1}^{t-1}-\xi_{s}^{h}\left(g^{h}, g^{h+1}, \mu_t,\zeta_{s}\right)+\ln \frac{\hat{p}_{t}^{\mu_t}\left(g^{h}, g^{h+1}\right)}{p_{0}^{h}\left(g^{h}\right) p_{0}^{h+1}\left(g^{h+1}\right)}\right] \\
&\qquad\geq \inf_{p} \mathbb{E}_{g \sim p}\left[\sum_{s=1}^{t-1}-\xi_{s}^{h}\left(g^{h}, g^{h+1},\mu_t, \zeta_{s}\right)+\ln \frac{p\left(g^{h}, g^{h+1}\right)}{p_{0}^{h}\left(g^{h}\right) p_{0}^{h+1}\left(g^{h+1}\right)}\right] \\
&\qquad=-\ln \mathbb{E}_{g^{h} \sim p_{0}^{h}} \mathbb{E}_{g^{h+1} \sim p_{0}^{h+1}} \exp \left(\sum_{s=1}^{t-1} \xi_{s}^{h}\left(g^{h}, g^{h+1},\mu_t, \zeta_{s}\right)\right),
\end{aligned}
$$
where the last step is from some simple calculations and the fact that Lemma~\ref{lem: Gibbs_var} implies that the $\inf$ is achieved by $p(g^h,g^{h+1}) \propto p_0^h(g^h)p_0^{h+1}(g^{h+1}) \exp(\sum_{s=1}^{t-1} \xi_{s}^{h}\left(g^{h}, g^{h+1},\mu_t, \zeta_{s}\right))$.

This implies that 
$$
\begin{aligned}
& \mathbb{E}_{S_{t-1}} \mathbb{E}_{f \sim \hat{p}_{t}} \Eb_{g \sim \hat{p}^{\mu_t}_t} \left[\sum_{s=1}^{t-1}-\xi_{s}^{h}\left(f^{h}, f^{h+1}, \mu_t,\zeta_{s}\right)+\ln \frac{\hat{p}^{\mu_t}_{t}\left(f^{h}, f^{h+1}\right)}{p_{0}^{h}\left(f^{h}\right) p_{0}^{h+1}\left(f^{h+1}\right)}\right] \\
&\qquad\geq -\mathbb{E}_{S_{t-1}} \mathbb{E}_{f \sim \hat{p}_{t}}\ln \mathbb{E}_{g^{h} \sim p_{0}^{h}} \mathbb{E}_{g^{h+1} \sim p_{0}^{h+1}} \exp \left(\sum_{s=1}^{t-1} \xi_{s}^{h}\left(g^{h}, g^{h+1},\mu_t, \zeta_{s}\right)\right) \\
&\qquad\geq -\ln \mathbb{E}_{g^{h} \sim p_{0}^{h}} \mathbb{E}_{g^{h+1} \sim p_{0}^{h+1}} \mathbb{E}_{f \sim \hat{p}_{t}} \mathbb{E}_{S_{t-1}} \exp \left(\sum_{s=1}^{t-1} \xi_{s}^{h}\left(g^{h}, g^{h+1},\mu_t, \zeta_{s}\right)\right)=0,
\end{aligned}
$$
where we use the convexity of $-\ln(\cdot)$. With this result, the definition of $B_h'$ and the definition of $\xi_{s}^{h}\left(g^{h}, g^{h+1}, \mu_t, \zeta_{s}\right)$, we have 
$$
\begin{aligned}
B_h' &=0.5 \alpha \mathbb{E}_{S_{t-1}} \mathbb{E}_{f_t \sim \hat{p}_{t}}\Eb_{g \sim \hat{p}_t^{\mu_t}}\left[\eta \sum_{s=1}^{t-1} 2 \Delta L^{h}_{\mu_t}\left(g^{h}, g^{h+1}, \zeta_{s}\right)+\ln \frac{\hat{p}^{\mu_t}_{t}\left(g^{h}, g^{h+1}\right)}{p_{0}^{h}\left(g^{h}\right) p_{0}^{h+1}\left(g^{h+1}\right)}\right]\\
&\geq 0.5\alpha \Eb_{S_{t-1}} \Eb_{f_t \sim \hat{p}_t} \Eb_{g \sim \hat{p}^{\mu_t}_t} \sum_{s=1}^{t-1}-\ln \mathbb{E}_{x_s^{h+1} \sim \Pb^h\left(\cdot \mid x_{s}^{h}, a_{s}^{h}, b_s^h\right)} \exp \left(-2 \eta \Delta L^{h}_{\mu_t}\left(g^{h}, g^{h+1}, \zeta_{s}\right)\right)\\
&\geq -0.5 \alpha \eta \sum_{s=1}^{t-1} \frac{1}{2}\mathbb{E}_{S_{t-1}} \mathbb{E}_{f \sim \hat{p}_{t}} \Eb_{g \sim \hat{p}^{\mu_t}_t} \mathbb{E}_{\pi_{s}} (\cE^{\mu_t}_h(g;x_s^h, a_s^h, r_s^h))^2,
\end{aligned}
$$
where we use Lemma~\ref{lem:ln_expectation2} in the last step. 

We now turn to the lower bound of $C_h$. For any max-player's policy $\mu$, we have 
$$
\begin{aligned}
& \mathbb{E}_{g \sim \hat{p}^\mu_{t}}\left[\alpha \ln \mathbb{E}_{\tilde{g}^{h} \sim p_{0}^{h}} \exp \left(-\eta \sum_{s=1}^{t-1} \Delta L^{h}_\mu \left(\tilde{g}^{h}, g^{h+1}, \zeta_{s}\right)\right)+(1-\alpha) \ln \frac{\hat{p}^\mu_{t}\left(g^{h+1}\right)}{p_{0}^{h+1}\left(g^{h+1}\right)}\right] \\
\geq &(1-\alpha) \inf_{p^h} \mathbb{E}_{g \sim p^h}\left[\frac{\alpha}{1-\alpha} \ln \mathbb{E}_{\tilde{g}^{h} \sim p_{0}^{h}} \exp \left(-\eta \sum_{s=1}^{t-1} \Delta L^{h}_\mu \left(\tilde{g}^{h}, g^{h+1}, \zeta_{s}\right)\right)+\ln \frac{p^h\left(g^{h+1}\right)}{p_{0}^{h+1}\left(g^{h+1}\right)}\right] \\
=&-(1-\alpha) \ln \mathbb{E}_{g^{h+1} \sim p_{0}^{h+1}}\left(\mathbb{E}_{g^{h} \sim p_{0}^{h}} \exp \left(-\eta \sum_{s=1}^{t-1} \Delta L^{h}_\mu\left(g^{h}, g^{h+1}, \zeta_{s}\right)\right)\right)^{-\alpha /(1-\alpha)},
\end{aligned}
$$
where we use the fact that the $\inf$ is achieved at 
$$
p^h\left(g^{h+1}\right) \propto p_{0}^{h+1}\left(g^{h+1}\right)\left(\mathbb{E}_{g^{h} \sim p_{0}^{h}} \exp \left(-\eta \sum_{s=1}^{t-1} \Delta L^{h}_\mu\left(g^{h}, g^{h+1}, \zeta_{s}\right)\right)\right)^{-\alpha /(1-\alpha)}.
$$
We now consider a fixed $g^h \in \cF_h^\mu(\epsilon, g^{h+1})$. Using the same arguments as in the proof of Lemma~\ref{lem:lower_bound}, it holds that 
$$
\left|\Delta L^{h}_\mu\left(g^{h}, g^{h+1}, \zeta_{s}\right)\right| \leq\left(\mathcal{E}^\mu_{h}\left(g, x_{s}^{h}, a_{s}^{h}\right)\right)^{2}+2 b\left|\mathcal{E}^\mu_{h}\left(g, x_{s}^{h}, a_{s}^{h}\right)\right| \leq \epsilon(2 b+\epsilon).
$$
It follows that 
$$
\mathbb{E}_{g^{h} \sim p_{0}^{h}} \exp \left(-\eta \sum_{s=1}^{t-1} \Delta L^{h}_\mu \left(g^{h}, g^{h+1}, \zeta_{s}\right)\right) \leq p_{0}^{h}\left(\mathcal{F}^\mu_{h}\left(\epsilon, g^{h+1}\right)\right) \exp (-\eta(t-1)(2 b+\epsilon) \epsilon)
$$
Thus, we have
\begin{align*}
C_h &\geq \alpha \Eb_{S_{t-1}} \Eb_{f \sim \hat{p}_t} \ln \Eb_{f^{h+1} \sim p_0^{h+1}} p_0^h(\cF^{\mu_t}_h(\epsilon, g^{h+1})) \exp(-\eta(t-1)(2b+\epsilon)\epsilon)\\
&= -\alpha\eta\epsilon(2b+\epsilon)(t-1) + \alpha \Eb_{S_{t-1}} \Eb_{f \sim \hat{p}_t} \ln \Eb_{g^{h+1} \sim p_0^{h+1}} p_0^h(\cF^{\mu_t}_h(\epsilon, g^{h+1})) \\
&\geq -\alpha\eta\epsilon(2b+\epsilon)(t-1) - \Eb_{S_{t-1}} \Eb_{f_t \sim \hat{p}_t} \kappa^h_{\mu_t}(\alpha, \epsilon),
\end{align*}
where we use the definition
$$\kappa^{h}_\mu(\alpha, \epsilon)=(1-\alpha) \ln \mathbb{E}_{g^{h+1} \sim p_{0}^{h+1}} p_{0}^{h}\left(\mathcal{F}^{\mu}_{h}\left(\epsilon, g^{h+1}\right)\right)^{-\alpha /(1-\alpha)}.$$
\end{proof}

We are ready to prove Theorem~\ref{thm:booster agent}.
\begin{proof}[Proof of Theorem~\ref{thm:booster agent}.]
Let $\pi_t$ denote the distribution induced by $\mu_t \times \nu_t$ and define
$$\delta_t^h = -\lambda \cE^{\mu_t}_h(g_t, x_t^h, a_t^h, b_t^h) - 0.25 \alpha \eta \sum_{s=1}^{t-1} \Eb_{\pi_s} \left(\cE^{\mu_t}_h(g_t, x_t^h, a_t^h, b_t^h)\right)^2.$$
According to the value-decomposition Lemma~\ref{lem:value_decom_booster agent}, we have
$$
\begin{aligned}
&\Eb_{S_{t-1}} \Eb_{f_t \sim \hat{p}_t}\Eb_{g_t \sim \hat{p}^{\mu_t}_t}  \lambda (V_1^{\mu_t, \nu_t}(x^1) - V_1^{\mu_t, \dagger}(x^1)) - \Eb_{S_{t-1}} \Eb_{f_t \sim \hat{p}_t}\Eb_{g_t \sim \hat{p}^{\mu_t}_t}  \Eb_{\zeta_t \sim \pi_t} \sum_{h=1}^H \delta_t^h\\
&= -\lambda \Eb_{S_{t-1}} \Eb_{f_t \sim \hat{p}_t}\Eb_{g_t \sim \hat{p}^{\mu_t}_t}  \Delta g^{\mu_t}_{t,1}(x^1) + 0.25\alpha \eta \sum_{h=1}^H \sum_{s=1}^{t-1} \Eb_{S_{t-1}} \Eb_{f_t \sim \hat{p}_t} \Eb_{g_t \sim \hat{p}^{\mu_t}_t} \Eb_{\pi_s} \left(\cE_h^{\mu_t}(g_t, x_t^h, a_t^h, b_t^h)\right)^2\\
&\leq \mathbb{E}_{S_{t-1}} \Eb_{f_t \sim \hat{p}_t} \mathbb{E}_{g_t \sim \hat{p}_{t}}\left(\sum_{h=1}^{H} \Phi_{t}^{h}(g_t, \mu_t)-\lambda \Delta g^{\mu_t}_{t,1}(x^1)+\ln \hat{p}^{\mu_t}_{t}(g_t)\right)\\
&\qquad+\alpha \eta \epsilon(2 b+\epsilon)(t-1) H+\Eb_{S_{t-1}} \Eb_{f_t \sim \hat{p}_t}\sum_{h=1}^{H} \kappa^{h}_{\mu_t}(\alpha, \epsilon)\\
&= \mathbb{E}_{S_{t-1}}\Eb_{f_t \sim \hat{p}_t} \inf _{p} \mathbb{E}_{g \sim p}\left(\sum_{h=1}^{H} \Phi_{t}^{h}(g, \mu_t)-\lambda \Delta g^{\mu_t}_{t,1}(x^1)+\ln p(g)\right) \\
&\qquad+\alpha \eta \epsilon(2 \beta+\epsilon)(t-1) H+\Eb_{S_{t-1}} \Eb_{f_t \sim \hat{p}_t} \sum_{h=1}^{H} \kappa^{h}_{\mu_t}(\alpha, \epsilon) \\
&\leq\lambda \epsilon+\alpha \eta \epsilon(\epsilon+4 \epsilon+2 \beta)(t-1) H-\Eb_{S_{t-1}} \Eb_{f_t \sim \hat{p}_t}\sum_{h=1}^{H} \ln p_{0}^{h}(\mathcal{F}_h( \epsilon, Q_{h+1}^{\mu_t,\dagger}))+\Eb_{S_{t-1}} \Eb_{f_t \sim \hat{p}_t}\sum_{h=1}^{H} \kappa^{h}_{\mu_t}(\alpha, \epsilon).
\end{aligned}
$$
Summing over $t$, we obtain that
$$
\begin{aligned}
&\sum_{t=1}^T \Eb_{S_{t-1}} \Eb_{f_t \sim \hat{p}_t} \mathbb{E}_{g_t \sim \hat{p}_{t}} (V_1^{\mu_t, \nu_t}(x^1) - V_1^{\mu_t, \dagger}(x^1))  \\
&\leq \epsilon T + \frac{1}{\lambda} \alpha \eta (5\epsilon + 2\beta) \frac{T(T-1)}{2} H - \frac{1}{\lambda}\sum_{t=1}^T \Eb_{S_{t-1}} \Eb_{f_t \sim \hat{p}_t}\sum_{h=1}^{H} \ln p_{0}^{h}(\mathcal{F}_h( \epsilon, Q_{h+1}^{\mu_t,\dagger})) \\
&+ \frac{1}{\lambda} \sum_{t=1}^T\Eb_{S_{t-1}} \Eb_{f_t \sim \hat{p}_t} \sum_{h=1}^{H} \kappa^{h}_{\mu_t}(\alpha, \epsilon) + \frac{\lambda}{\alpha \eta} dc(\cF, MG, T, 0.25\alpha \eta /\lambda)\\
&\leq O(\beta\sqrt{dc(\cF,MG,T)\kappa(\beta/T^2)T} + dc(\cF,MG,T)). 
\end{aligned}
$$
The last step is proved as follows. We find an upper bound for $\Eb_{S_{t-1}} \Eb_{f_t \sim \hat{p}_t}\sum_{t=1}^T \sum_{h=1}^{H} \kappa^{h}_{\mu_t}(\alpha, \epsilon)$. We note that for all $\mu_t$, $\kappa_{\mu_t}^h(\alpha, \epsilon)$ is increasing w.r.t. $\alpha$ with the limit $\kappa_{\mu_t}^h(1,\epsilon) \leq \kappa(\epsilon) < \infty$. By monotone convergence theorem, we know that 
$$\Eb_{S_{t-1}} \Eb_{f_t \sim \hat{p}_t}\sum_{t=1}^T \sum_{h=1}^{H} \kappa^{h}_{\mu_t}(\alpha, \epsilon) \to \Eb_{S_{t-1}} \Eb_{f_t \sim \hat{p}_t}\sum_{t=1}^T \sum_{h=1}^{H} \kappa^{h}_{\mu_t}(1, \epsilon) = \sum_{t=1}^T\Eb_{S_{t-1}} \Eb_{f_t \sim \hat{p}_t} \kappa_{\mu_t}(\epsilon) \leq T\kappa(\epsilon).$$
We also have
$$\Eb_{S_{t-1}} \Eb_{f_t \sim \hat{p}_t}\sum_{t=1}^T \sum_{h=1}^{H} -\ln p_{0}^{h}\left(\mathcal{F}_h{(} \epsilon, Q_{h+1}^{\mu_t,\dagger})\right) \leq \Eb_{S_{t-1}} \Eb_{f_t \sim \hat{p}_t}\sum_{t=1}^T  \kappa_{\mu_t}(\epsilon) \leq T\kappa(\epsilon).$$
It follows that 
$$
\begin{gathered}
- \sum_{t=1}^T \Eb_{S_{t-1}} \Eb_{f_t \sim \hat{p}_t}\sum_{h=1}^{H} \ln p_{0}^{h}(\mathcal{F}_h( \epsilon, Q_{h+1}^{\mu_t,\dagger})) + \sum_{t=1}^T \Eb_{S_{t-1}} \Eb_{f_t \sim \hat{p}_t}\sum_{h=1}^{H} \kappa^{h}_{\mu_t}(1, \epsilon) \leq 2T\kappa(\epsilon).
\end{gathered}
$$
Now we first let $\alpha \to 1^{-}$. Then, we take $\epsilon = \frac{\beta}{T^2}$, $\lambda = \sqrt{\frac{T\kappa(\beta/T^2)}{\beta^2dc(\cF,MG,T)}}$, $\eta  = \frac{1}{4\beta^2}$. This concludes the proof.
\end{proof}

\section{Proof of the Value-Decomposition Lemma} \label{sec:value_decom}
\begin{proof}[Proof of Lemma~\ref{lem:value_decom_main}]
Let $\mu = \mu_f$ and $\nu$ be an arbitrary policy taken by the min-player.
$$
\begin{aligned}
&V^*_1(x^1) - V_1^{\mu, \nu}(x^1) \\
&\qquad= \sum_{h=1}^H \Eb_{\mu, \nu} V_{f,h}(x^h) - r^h(x^h, a^h, b^h) - V_{f, h+1}(x^{h+1}) + V^*_1(x^1)-V_{f,1}(x^1)\\
&\qquad= \sum_{h=1}^H \Eb_{\mu, \nu}  \min_{\nu'} \Db_{\mu, \nu'} f(x^h) - r^h(x^h, a^h, b^h) - V_{f, h+1}(x^{h+1}) + V^*_1(x^1)-V_{f,1}(x^1)\\
&\qquad\leq \sum_{h=1}^H \Eb_{\mu, \nu}  \Db_{\mu, \nu} f^h(x^h) - r^h(x^h, a^h, b^h) - V_{f, h+1}(x^{h+1}) + V^*_1(x^1)-V_{f,1}(x^1)\\
&\qquad= \sum_{h=1}^H \Eb_{\mu, \nu} f^h(x^h, a^h, b^h) - r^h(x^h, a^h, b^h) - V_{f, h+1}(x^{h+1})+V^*_1(x^1)-V_{f,1}(x^1)\\
&\qquad= \sum_{h=1}^H \Eb_{\mu, \nu} \cE_h(f^h, f^{h+1}, \zeta) +V^*_1(x^1)-V_{f,1}(x^1),
\end{aligned}
$$
where the first equality comes from the value-decomposition Theorem \cite{jiang@2017} (can be verified easily by telescope sum and $V^{H+1}=0$); the second equality is because of the definition of $\mu = \mu_{f, h}(x)=\underset{\mu \in \Delta_{\mathcal{A}}}{\operatorname{argmax}} \min _{\nu \in \Delta_{\mathcal{B}}} \mu^{\top} f^{h}(x, \cdot, \cdot) \nu
$; the inequality comes from the fact that $\mu = \mu_f$ and $\nu$ may not be $\argmin_{\nu'} \Db_{\mu, \nu'} f(x^h)$. {This decomposition accounts for the use of an optimistic prior in Algorithm~\ref{alg:main}}.
\end{proof}

\begin{proof}[Proof of Lemma~\ref{lem:value_decom_booster agent}]
Suppose that $\mu = \mu_f$ is taken by the max-player and $g$ is sampled from the posterior by the booster agent. Let $\nu$ be given by $\nu = \argmin_{\nu'} V^\mu_h(x)$ for all $(x,h)$. Then, we have:
$$
\begin{aligned}
&V_1^{\mu, \dagger}(x^1) - V_1^{\mu, \nu}(x^1) \\
&\qquad= V_{g,1}^{\mu}(x^1) - V_1^{\mu, \nu}(x^1) + V_1^{\mu_t, \dagger}(x^1) -V_{g,1}^{\mu}(x^1)\\
&\qquad= \sum_{h=1}^H \Eb_{\mu, \nu} \Db_{\mu, \nu} g(x^h) - r^h(x^h, a^h, b^h) - V^\mu_{g, h+1}(x^{h+1})+ V_1^{\mu, \dagger}(x^1) -V_{g,1}^{\mu}(x^1)\\
&\qquad= \sum_{h=1}^H \Eb_{\mu, \nu} g^h(x^h, a^h, b^h) - r^h(x^h, a^h, b^h) - V^\mu_{g, h+1}(x^{h+1})+ V_1^{\mu, \dagger}(x^1) -V_{g,1}^{\mu}(x^1)\\
&\qquad= \sum_{h=1}^H \Eb_{\mu, \nu} \cE^\mu_h(g^h, g^{h+1}, \zeta) + V_1^{\mu, \dagger}(x^1) -V_{g,1}^{\mu}(x^1).
\end{aligned}
$$
This decomposition accounts for the use of an optimistic prior in Algorithm~\ref{alg:booster agent}.
\end{proof}

\section{Proof of the Decoupling Coefficient Bounds} \label{appen:dc_bounds}
In this section, we provide proofs for the decoupling coefficient bounds. We need the following lemma.
\begin{lemma}[Elliptical Potential Lemma, Lemma $10$ of \citet{xie2020learning}] Suppose $\{\phi_t\}_{t\geq 0}$ is a sequence in $\RR^d$ satisfying $\norm{\phi_t} \leq 1$. Let $\Lambda_0 \in \RR^{d\times d}$ be a positive definite matrix, and $\Lambda_t = \Lambda_0 + \sum_{i=1}^t \phi_i\phi_i^\top$. If the smallest eigenvalues of $\Lambda_0$ is lower bounded by $1$, then 
$$
\log \left(\frac{\operatorname{det} \Lambda_{t}}{\operatorname{det} \Lambda_{0}}\right) \leq \sum_{i \in[t]} \phi_{i}^{\top} \Lambda_{j-1}^{-1} \phi_{i} \leq 2 \log \left(\frac{\operatorname{det} \Lambda_{t}}{\operatorname{det} \Lambda_{0}}\right).
$$
\end{lemma}
\begin{proof}[Proof of Proposition~\ref{prop:linear}] 
We first note that the completeness assumption is satisfied in linear MG case whose proof can be found in \citet{huang2021towards}. Now we consider two arbitrary ${\theta^h}, \theta^{h+1}$ whose norms are bounded by $H\sqrt{d}$ and $f\in \cF$. We also define a function $g \in \cF$ s.t. $g^h = g({\theta^h})$ and $g^{h+1} = g({\theta^{h+1}})$.  By Assumption~\ref{assu:completeness}, we can find some $\theta^h(f) \in \RR^d$ with $\norm{\theta^h(f)} \leq H\sqrt{d}$ s.t. $\cT^{\mu_f}_h(\phi(x,a,b)^\top \theta^{h+1}) = \phi(x,a,b)^\top \theta^h(f)$. Therefore, we have
$$\cE_h^{\mu_f}(g;x,a,b) = \phi(x,a,b)^\top(\theta^h - \theta^h(f)) = \phi(x,a,b)^\top w^h(f,g),$$
where $w^h(f,g) \in \RR^d$ satisfies $\norm{w^h(f,g)} \leq 2H\sqrt{d}$. We denote $\phi^h_s = \Eb_{\pi_s} [\phi(x^h,a^h,b^h)]$ and $\Phi_t^h = \lambda I + \sum_{s=1}^t \phi(x^h,a^h,b^h) \phi(x^h,a^h,b^h)^\top$ where $\lambda \geq 1$ is a tuning parameter. Then, we have
$$
\begin{aligned}
&\mathbb{E}_{\pi_{{t}}}\left[\mathcal{E}^{\mu_{f_t}}_{h}\left(g_{t} ; x_{t}^{h}, a_{t}^{h}, b_t^h\right)\right]-\mu \sum_{s=1}^{t-1} \mathbb{E}_{\pi_{{s}}}\left[\mathcal{E}_{h}^{\mu_{f_t}}\left(g_{t} ; x_{s}^{h}, a_{s}^{h}\right)^{2}\right] \\
&=w^{h}\left(f_{t},g_t\right)^{\top} \phi_{t}^{h}-\mu w^{h}\left(f_{t},g_t\right)^{\top} \sum_{s=1}^{t-1} \mathbb{E}_{\pi_{{s}}}\left[\phi\left(x^{h}, a^{h}, b^h\right) \phi\left(x^{h}, a^{h}, b^h\right)^{\top}\right] w^{h}\left(f_{t},g_t\right) \\
&\leq w^{h}\left(f_{t}, g_t\right)^{\top} \phi_{t}^{h}-\mu w^{h}\left(f_{t}, g_t\right)^{\top} \Phi_{t-1}^{h} w^{h}\left(f_{t}, g_t\right)+4\mu \lambda dH^2 \\
&\leq \frac{1}{4 \mu}\left(\phi_{t}^{h}\right)^{\top}\left(\Phi_{t-1}^{h}\right)^{-1} \phi_{t}^{h}+4\mu \lambda dH^2
\end{aligned}
$$
where the first inequality uses Jensen's inequality and $\norm{w^h(f_t,g_t)} \leq  2H\sqrt{d}$ and the second inequality comes from the fact $(a^\top b) \leq (\norm{a}_{\Phi_{t-1}^h} \norm{b}_{(\Phi_{t-1}^h)^{-1}}) \leq \frac{1}{2}(\norm{a}^2_{\Phi_{t-1}^h} + \norm{b}^2_{(\Phi_{t-1}^h)^{-1}})$. Summing over $t \in [T]$ and $h \in [H]$, we have
$$
\begin{aligned}
&\sum_{t=1}^T \sum_{h=1}^H \mathbb{E}_{\pi_{{t}}}\left[\mathcal{E}^{\mu_{f_t}}_{h}\left(g_{t} ; x_{t}^{h}, a_{t}^{h}, b_t^h\right)\right]-\mu \sum_{s=1}^{t-1} \mathbb{E}_{\pi_{{s}}}\left[\mathcal{E}_{h}^{\mu_{f_t}}\left(g_{t} ; x_{s}^{h}, a_{s}^{h}\right)^{2}\right] \\
&\leq \sum_{h=1}^H \left[\frac{\ln(\text{det}(\Phi_T^h)-d\ln(\lambda)}{2\mu} + 4\mu \lambda dH^2 T\right] \\
&\leq H(\frac{d\ln(\lambda + T/d) - d\ln(\lambda)}{2\mu} + 4\mu \lambda dH^2 T),
\end{aligned}
$$
where we use the Elliptical Potential lemma in the first inequality and the second inequality uses 
$$\ln \text{det}(\Phi_T^h) \leq d\ln \frac{\text{trace}(\Phi_T^h)}{d}, \text{ and, }\text{trace}(\Phi_t^h) \leq \lambda d + T.$$
By setting $\lambda = \min\{1, \frac{1}{\mu^2H^2T}\}$, we conclude the proof.
\end{proof}

\begin{proof}[Proof of Proposition~\ref{prop:gen_linear}]
We assume that $c_1 \leq 1 \leq c_2$. Otherwise, we can scale the feature maps and the link function accordingly. By similar arguments with the completeness assumption as in the proof of Proposition~\ref{prop:linear}, we have
$$\cE_h^{\mu_f}(g;x,a,b) = \sigma(\phi(x,a,b)^\top\theta^h) - \sigma(\phi(x,a,b)^\top\theta^h(f)).$$
By the Lipschitz property, we have 
$$c_1|\phi(x,a,b)^\top w(f,g)| \leq |\cE^{\mu_f}_h(g;x,a,b)| \leq c_2 |\phi(x,a,b)^\top w(f,g)|,$$
for some $w(f,g) \in \RR^d$ satisfying $w(f,g) \leq  2H\sqrt{d}$. We denote $\phi_h^s = \Eb_{\pi_s} [\phi(x^h,a^h,b^h)]$ and $\Phi_t^h = \lambda I + \sum_{s=1}^t \phi(x^h,a^h,b^h) \phi(x^h,a^h,b^h)^\top$ where $\lambda \geq 1$ is a tuning parameter. Then, we have
$$
\begin{aligned}
&\mathbb{E}_{\pi_{{t}}}\left[\mathcal{E}^{\mu_{f_t}}_{h}\left(g_{t} ; x_{t}^{h}, a_{t}^{h}, b_t^h\right)\right]-\mu \sum_{s=1}^{t-1} \mathbb{E}_{\pi_{{s}}}\left[\mathcal{E}_{h}^{\mu_{f_t}}\left(g_{t} ; x_{s}^{h}, a_{s}^{h}\right)^{2}\right] \\
&\leq c_2|w^{h}\left(f_{t},g_t\right)^{\top} \phi_{t}^{h}|-\mu c_1^2 w^{h}\left(f_{t},g_t\right)^{\top} \sum_{s=1}^{t-1} \mathbb{E}_{\pi_{{s}}}\left[\phi\left(x^{h}, a^{h}, b^h\right) \phi\left(x^{h}, a^{h}, b^h\right)^{\top}\right] w^{h}\left(f_{t},g_t\right) \\
&\leq c_2|w^{h}\left(f_{t}, g_t\right)^{\top} \phi_{t}^{h}|-\mu c_1^2 w^{h}\left(f_{t}, g_t\right)^{\top} \Phi_{t-1}^{h} w^{h}\left(f_{t}, g_t\right)+4\mu c_1^2 \lambda dH^2 \\
&\leq \frac{c_2^2}{4 \mu c_1^2}\left(\phi_{t}^{h}\right)^{\top}\left(\Phi_{t-1}^{h}\right)^{-1} \phi_{t}^{h}+4\mu c_1^2\lambda dH^2.
\end{aligned}
$$
Summing over $t \in [T]$ and $h \in [H]$, we have
$$
\begin{aligned}
&\sum_{t=1}^T \sum_{h=1}^H \mathbb{E}_{\pi_{{t}}}\left[\mathcal{E}^{\mu_{f_t}}_{h}\left(g_{t} ; x_{t}^{h}, a_{t}^{h}, b_t^h\right)\right]-\mu \sum_{s=1}^{t-1} \mathbb{E}_{\pi_{{s}}}\left[\mathcal{E}_{h}^{\mu_{f_t}}\left(g_{t} ; x_{s}^{h}, a_{s}^{h}\right)^{2}\right] \\
&\leq \sum_{h=1}^H c_2^2 \left[\frac{\ln(\text{det}(\Phi_T^h)-d\ln(\lambda)}{2\mu c_1^2} + 4\mu \lambda c_1^2 H^2d T\right] \\
&\leq Hc_2^2(\frac{d\ln(\lambda + T/d) - d\ln(\lambda)}{2\mu c_1^2} + 4\mu \lambda c_1^2 H^2d T).
\end{aligned}
$$
Setting $\lambda = \min\{1, \frac{1}{\mu^2c_1^2H^2T}\}$ concludes the proof.
\end{proof}

In what follows, we prove the reduction of Bellman-Eluder dimension to the decoupling coefficient following the analysis of \citet{dann2021provably}. From a high level, the multi-agent decoupling coefficient bounds the \textit{out-of-sample} prediction error by the \textit{in-sample} error. In particular, we remark that the existing techniques in the literature of Eluder dimension (e.g., Lemma 41 of \citet{jin2021bellman}) is not sufficient for our needs. This is because a deterministic upper bound of the in-sample error, i.e., the confidence radius, is not available for the posterior sampling. We start the following lemma from \citet{dann2021provably}.

\begin{lemma} \label{lem:bed_proof} For any sequence of positive reals $x_1,\cdots,x_n$, we have 
$$
f(x):=\frac{\sum_{i=1}^{n} x_{i}}{\sqrt{\sum_{i=1}^{n} i x_{i}^{2}}} \leq \sqrt{1+\log (n)}.
$$
\end{lemma}

\begin{proof}[Proof of Proposition~\ref{prop:reduction_be}]
We consider a fixed $h \in [H]$. We first introduce some short-hand notations. We denote $\hat{\epsilon}_{st}^h = |\Eb_{\mu_s} \cE_h^{\mu_{f_t}}(g_t,x^h,a^h)|$ and $\epsilon^h_{st} = \hat{\epsilon}_{st}^h \mathbf{1}(\hat{\epsilon}_{st}^h > \epsilon)$. The proof proceeds as follows. We initialize $T$ empty buckets $B_0^h,\cdots, B_{T-1}^h$ and go through $\epsilon_{tt}^h$ one by one for $t \in [T]$.

If $\epsilon_{tt}^h = 0$, we discard the timestep. Otherwise, we go through the buckets from $0$ in increasing order. At bucket $i$, 
\begin{itemize}
    \item if $\sum_{s \leq t-1, s \in B_i^h} (\epsilon_{st}^h)^2 < (\epsilon_{tt}^h)^2$, we add $t$ into $B_i^h$;
    \item otherwise, we continue with the next bucket.
\end{itemize}
We denote the index of bucket that each non-zero timestep ends up in as $b_t^h$. As $\epsilon_{tt}^h$ skip the bucket $0,\cdots,b_t^h-1$, by construction, we have
$$\begin{aligned}
\sum_{t=1}^{T} \sum_{s=1}^{t-1}\left(\epsilon_{s t}^{h}\right)^{2} \geq \sum_{t=1}^T \sum_{0\leq i \leq b_t^h-1} \sum_{s \leq t-1, s \in B_i^h} \left(\epsilon_{st}^h\right)^2 \geq \sum_{t=1}^{T} b_{t}^{h}\left(\epsilon_{t t}^{h}\right)^{2}.
\end{aligned}
$$
Note by the definition of Eluder dimension, for the measures in $B_i^h$, say, $\{\mu_{t_i}: i = 1,\cdots,m\}$, $\mu_{t_j}$ is $\epsilon$-independent from all the predecessors $\mu_{t_1},\cdots,\mu_{t_{j-1}}$. Therefore, the size of each bucket cannot exceed the Bellman Eluder dimension $E_\epsilon = \mathrm{dim_{BE}}(\cF, \Pi, \epsilon)$. By Jensen's inequality, we can obtain that
$$
\begin{aligned}
\sum_{t=1}^{T} b_{t}^{h}\left(\epsilon_{t t}^{h}\right)^{2} &=  \sum_{i=1}^{T-1} i \sum_{s \in B_i^h} (\epsilon_{ss}^h)^2 \geq \sum_{i=1}^{T-1} i |B_i^h| \left(\sum_{s \in B_i^h} \frac{\epsilon_{ss}^h}{|B_i^h|}\right)^2\geq \sum_{i=1}^{T-1} i E_\epsilon \left(\sum_{s \in B_i^h} \frac{\epsilon_{ss}^h}{E_\epsilon}\right)^2,
\end{aligned}
$$
where the last inequality uses $|B_i^h| \leq E_\epsilon$. By Lemma~\ref{lem:bed_proof} with $x_i = \sum_{s \in B_i^h} \epsilon_{ss}^h$, we know that 
$$
\begin{aligned}
\sum_{i=1}^{T-1} E_\epsilon i\left(\sum_{s \in B_{i}^{h}} \frac{\epsilon_{s s}^{h}}{E_{\epsilon}}\right)^{2}&=\frac{1}{E_{\epsilon}} \sum_{i=1}^{T-1} i\left(\sum_{s \in B_{i}^{h}} \epsilon_{s s}^{h}\right)^{2} \geq \frac{1}{E_{\epsilon}(1+\log (T))}\left(\sum_{i=1}^{T-1} \sum_{s \in B_i^h} \epsilon_{s s}^{h}\right)^{2}\\
&= \frac{1}{E_{\epsilon}(1+\log (T))}\left(\sum_{s \in[T] \backslash B_{0}^{h}} \epsilon_{s s}^{h}\right)^{2}.
\end{aligned}
$$
To summarize, we have proved that 
$\sum_{t=1}^{T} \sum_{s=1}^{t-1}\left(\epsilon_{s t}^{h}\right)^{2} \geq  \frac{1}{E_{\epsilon}(1+\log (T))}\left(\sum_{s \in[T] \backslash B_{0}^{h}} \epsilon_{s s}^{h}\right)^{2}.$ It follows that
$$
\begin{aligned}
\sum_{h=1}^{H} \sum_{t=1}^{T} \hat{\epsilon}_{t t}^{h} &\leq HT\epsilon +\sum_{h=1}^{H} \sum_{t=1}^{T} \epsilon_{t t}^{h} \leq HT \epsilon + E_\epsilon H+\sum_{h=1}^{H} \sum_{s \in[T] \backslash B_{0}^{h}} \epsilon_{s s}^{h}\\
&\leq HT \epsilon + E_\epsilon H+\sum_{h=1}^{H} \sqrt{\left(E_\epsilon (1 + \log (T))\right) \sum_{t=1}^{T} \sum_{s=1}^{t-1}\left(\epsilon_{s t}^{h}\right)^{2}}\\
%&\leq HT \epsilon + E_\epsilon H+ \sqrt{\left(E_\epsilon H ({1 + \log (T)})\right) \sum_{t=1}^{T} \sum_{h=1}^H\sum_{s=1}^{t-1}\left(\epsilon_{s t}^{h}\right)^{2}}.
&\leq HT \epsilon + E_\epsilon H+ \mu \sum_{h=1}^H \sum_{t=1}^T \left[\sum_{s=1}^{t-1} (\epsilon_{st}^h)^2\right]+ \frac{(1+\log(T)) E_\epsilon H}{4 \mu}\\
%\left(E_\epsilon H ({1 + \log (T)})\right) \sum_{t=1}^{T} \sum_{h=1}^H\sum_{s=1}^{t-1}\left(\epsilon_{s t}^{h}\right)^{2}.
&\leq \mu \sum_{h=1}^H \sum_{t=1}^T \left[\sum_{s=1}^{t-1} (\epsilon_{st}^h)^2\right]+ \frac{(1+\log(T) + 8\mu) E_\epsilon H}{4 \mu}
\end{aligned}
$$
where the second inequality follows from $|B_0^h| \leq E_\epsilon$ and the last inequality follows from $\epsilon = 1/T$. 
\end{proof}
%%%%%%%%%%%%%%%%%%%%%%%%%%%%%%%%%%%%%%%%%%%%%%%%%%%%%%%%%%%%%%%%%%%%%%%%%%%%%%%
%%%%%%%%%%%%%%%%%%%%%%%%%%%%%%%%%%%%%%%%%%%%%%%%%%%%%%%%%%%%%%%%%%%%%%%%%%%%%%%

\end{document}